\RequirePackage{fix-cm}
\documentclass[twocolumn]{svjour3}          
\smartqed  

\usepackage{graphicx}
\usepackage{amsmath}
\usepackage{amssymb}
\usepackage{subfig}
\usepackage{algorithm2e}  
\usepackage{algorithmic}
\usepackage{url}

\newcommand{\BL}[1]{\textit{\textcolor{blue}}{#1}}
\newcommand{\BF}[1]{\textbf{#1}}

\newcommand{\rp}{\texttt{rank-pool}}
\newcommand{\IDT}{\texttt{IDT}}

\usepackage{mathptmx}      
\usepackage{latexsym}

\journalname{International Journal of Computer Vision}
\usepackage{xspace}
\newcommand*{\eg}{e.g.\@\xspace}
\newcommand*{\ie}{i.e.\@\xspace}
\newcommand*{\etal}{et~al.\@\xspace}

\makeatletter
\newcommand*{\etc}{%
    \@ifnextchar{.}%
        {etc}%
        {etc.\@\xspace}%
}
\makeatother
\usepackage{sg-macros}
\newtheorem{defn}[THEOREM]{Definition}
\begin{document}

\title{Discriminatively Learned Hierarchical Rank Pooling Networks
}


\author{Basura~Fernando~\and~Stephen~Gould}
\authorrunning{Fernando and Gould} 

\institute{ACRV, Research School of Engineering \\
The Australian National University\\
              Canberra, Australia\\
              Tel.: +61-2612-53973\\              
              \email{basura.fernando@anu.edu.au}}

\date{Received: date / Accepted: date}

\maketitle

\begin{abstract}
Rank pooling is a temporal encoding method that summarizes the dynamics of a video sequence to a single vector which has shown good results in human action recognition in prior work. 
In this work, we present novel temporal encoding methods for action and activity classification by extending the unsupervised rank pooling temporal encoding method in two ways.

First, we present \emph{discriminative rank pooling} in which the shared weights of our video representation and the parameters of the action classifiers are estimated jointly for a given training dataset of labelled vector sequences using a bilevel optimization formulation of the learning problem. 
When the frame level features vectors are obtained from a convolutional neural network (CNN), we rank pool the network activations and jointly estimate all parameters of the model, including CNN filters and fully-connected weights, in an end-to-end manner which we coined as \emph{end-to-end trainable rank pooled CNN}. 
Importantly, this model can make use of any existing convolutional neural network architecture (e.g., AlexNet or VGG) without modification or introduction of additional parameters.

Then, we extend rank pooling to a high capacity video representation, called \emph{hierarchical rank pooling}.
Hierarchical rank pooling consists of a network of rank pooling functions, which encode temporal semantics over arbitrary long video clips based on rich frame level features. 
By stacking non-linear feature functions and temporal sub-sequence encoders one on top of the other, we build a high capacity encoding network of the dynamic behaviour of the video. 
The resulting video representation is a fixed-length feature vector describing the entire video clip that can be used as input to standard machine learning classifiers.

We demonstrate our approach on the task of action and activity recognition. We present a detailed analysis of our approach against competing methods and explore variants such as hierarchy depth and choice of non-linear feature function. 
Obtained results are comparable to state-of-the-art methods on three important activity recognition benchmarks with classification performance of 76.7\% mAP on Hollywood2, 69.4\% on HMDB51, and 93.6\% on UCF101.


\keywords{rank pooling \and action recognition \and activity recognition \and convolutional neural networks}
\end{abstract}

\section{Introduction}
\label{sec:intro}
\begin{figure}
 \centering 
 \includegraphics[width=\linewidth]{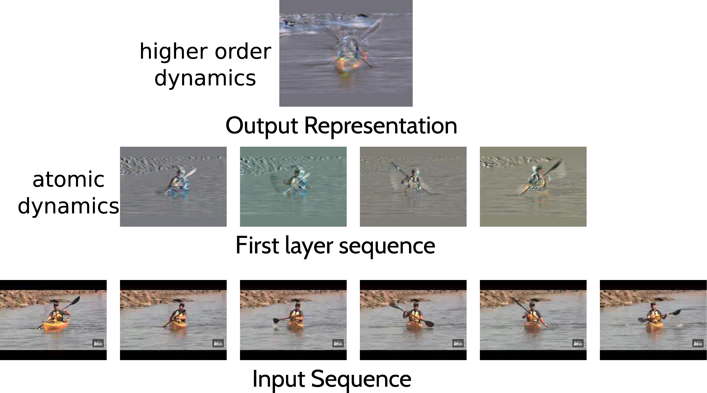}
 \caption{Illustration of hierarchical rank pooling for encoding the temporal
 dynamics of a video sequence.}
 \label{fig:illustrate} 
\end{figure}

Representation learning from sequence data has many applications including action and activity recognition from videos \cite{Poppe2010}, gesture recognition \cite{Bregler1997}, music classification from audio clips \cite{Lu2002}, and gene regulatory network analysis from gene expressions \cite{Shinozaki2003}.
In this paper we focus on activity and action recognition in videos, which is important for many real life applications including human computer interaction, sports analytic, and elderly monitoring and healthcare.
Neural network-based supervised learning of representations from sequence data has many advantages compared to hand-crafted feature engineering. 
However, capturing the discriminative behaviour of sequence data is a very challenging problem; especially when neural network-based supervised learning is used, which can overfit to irrelevant temporal signals.
In video sequence classification, and especially in action recognition, a key challenge is to obtain discriminative video representations that generalize beyond the training data.
Moreover, a good video representation should be invariant to the speed of the human actions and should be able to capture long term time evolution information, i.e.,~the temporal dynamics.
In action recognition a key challenge is to extract and represent high-level motion patterns, dynamics, and evolution of appearance of videos.
One can argue that end-to-end learning of video representations are the key to successful human action recognition.
However, it is extremely hard problem due to massive amount of video data that is required to learn such end-to-end video representations. 
A further challenge is to encode dynamics efficiently and effectively from variable length sequences.
This calls for novel spatio-temporal neural network architectures.

Recent success in action and activity recognition has been achieved by modelling evolving temporal dynamics in video sequences \cite{Bilen2016,Fernando2015,Fernando2016,karpathy2014large,Srivastava2015,Yue-HeiNg2015}. 
Some methods use linear ranking machines to capture first order dynamics \cite{Fernando2015,hoi2014}. 
Other methods encode temporal information using RNN-LSTMs on video sequences \cite{Srivastava2015,Yue-HeiNg2015,Zha2015}, but at the cost of many more model parameters.
To further advance activity recognition it is beneficial to exploit temporal information at multiple levels of granularity in a hierarchical manner and thereby capture more complex dynamics of the input sequences \cite{Du2015,Lan2015b,Song2013}. 
As frame based features improve, \eg, from a convolutional neural network (CNN), it is important to exploit information not only in the spatial domain but also in the temporal domain.
Several recent methods have obtained significant improvements in image categorisation and object detection using very deep CNN architectures \cite{Simonyan2014a}. 
Motivated by these deep hierarchies \cite{Du2015,Lan2015b,Song2013,Simonyan2014a}, we argue that learning a temporal encoding at a single level is not sufficient to interpret and understand video sequences, and that a temporal hierarchy is needed. 

In addition, we argue that end-to-end learning of video representations are necessary for reliable human action recognition. 
In recent years CNNs have become very popular for automatically learning representations from large collections of static images. 
Many tasks in computer vision, such as image classification, image segmentation and object detection, have benefited from such automatic representation learning \cite{Krizhevsky2012,Girshick2014}. 
However, it is unclear how one may extend these highly successful CNNs to sequence data; especially, when the intended task requires capturing dynamics of video sequences (e.g.,~action and activity recognition). 
Indeed, capturing the discriminative dynamics of a video sequence remains an open problem. 
Some authors have proposed to use recurrent neural networks (RNNs) \cite{Du2015} or extensions, such as long short term memory (LSTM) networks \cite{Srivastava2015}, to classify video sequences. 
However, CNN-RNN/LSTM models introduce a large number of additional parameters to capture sequence information. 
Consequently, these methods need much more training data. 
For sequence data such as videos, obtaining labelled training data is significantly more costly than obtaining labels for static images.
This is reflected in the size of datasets used in action and activity recognition research today. 
Even though there are datasets that consist of millions of labelled images (e.g.,~ImageNet~\cite{ImageNet:2009}), the largest fully labelled action recognition dataset, UCF101, consists of barely more than 13,000 videos \cite{soomro2012ucf101}. 
Some notable efforts to create large action recognition datasets include the Sports-1M~\cite{karpathy2014large}, the YouTube-8M~\cite{Abu-El-Haija2016} and the ActivityNet dataset~\cite{Snoek2016}.
The limitation of Sports-1M and YouTube-8M is that they are constructed from weakly labelled human annotations and sometimes annotations are very noisy.
Furthermore, ActivityNet only consist of 20,000 high quality annotated videos, which is insufficient for learning good video representations.
Despite recent efforts in building good action recognition datasets~\cite{kay2017kinetics}, it is highly desirable, therefore, to develop frameworks that can learn discriminative dynamics from video data without the cost of additional training data or model complexity.

Perhaps the most straightforward CNN-based method for encoding video sequence data is to apply temporal max pooling or temporal average pooling over the video frames. 
However, these methods do not capture any valuable time varying information of the video sequences \cite{karpathy2014large}. 
In fact, an arbitrary reshuffling of the frames would produce an identical video representation under these pooling schemes. 
Rank-pooling \cite{Fernando2015,Fernando2016}, on the other hand, attempts to encode time varying information by learning a linear ranking machine, one for each video, to produce a chronological ordering of the video's frames based on their appearance (i.e.,~the hand-crafted or CNN features). 
The parameters of the ranking machine (i.e.,~fit linear model) are then used as the video representation. 
However, unlike max and average pooling, it was previously unclear how the CNN parameters can be fine-tuned to give a more discriminative representation when rank-pooling is used since there is no closed-form formula for the rank-pooling operation and the derivative of its input arguments with respect to the rank-pool output not obvious.

The original rank pooling method of Fernando \etal \cite{Fernando2015,Fernando2016} obtained good activity recognition performance using hand-crafted features.
Given a sequence of video frames, the rank pooling method returns a vector of parameters encoding the dynamics of that sequence. 
The vector of parameters is derived from the solution of a linear ranking SVM optimization problem applied to the entire video sequence, \ie, at a single level.
We extend that work in two important directions that facilitates the use of richer CNN-based features to describe the input frames and allows the processing of more complex video sequences.


First, we show how to learn discriminative dynamics of video sequences or vector sequences using rank pooling-based temporal pooling. We show how the parameters of the activity classifier, shared parameters of video representations, and the CNN features themselves can all be learned jointly using a principled optimization framework. A key technical challenge, however, is that the optimization problem contains rank pooling as a subproblem---itself a non-trivial optimization problem. 
This leads to a large-scale bilevel optimization problem \cite{Bard} with convex inner-problem, which we propose to solve by stochastic gradient descent. The result is a higher capacity model than Fernando \etal \cite{Fernando2015,Fernando2016}, which is tuned to produce features that are discriminative for the task at-hand.
Concisely, we learn \emph{discriminative dynamics} during learning by propagating back the errors from the final classification layer to learn both video representation and a good classifier.

Second, we propose a hierarchical rank-pooling scheme that encodes a video sequence at multiple levels. The original video sequence is divided into multiple overlapping video segments. At the lowest level, we encode each video segment using rank pooling to produce a sequence of descriptors, one for each segment, which captures the dynamics of the small video segments (see Figure~\ref{fig:illustrate}). We then take the resulting sequence, divide that into multiple subsequences, and apply rank pooling to each of these next-level subsequences.
By recursively applying rank pooling on the obtained segment descriptors from the previous layer, we capture higher-order, non-linear, and more complex dynamics as we move up the levels of the hierarchy.
The final representation of the video is obtained by encoding the top-level dynamic sequence using yet one more rank pooling. 
This strategy allows us to encode more complicated activities thanks to the higher capacity of the model. 
In summary, our proposed hierarchical rank pooling model consists of a feed forward network starting with a frame-based CNN and followed by a series of point-wise non-linear operations and rank pooling operations over subsequences as illustrated in Figure~\ref{fig:hpooling}.

Our main contributions are then: (1) a novel discriminative dynamics learning framework in which we learn discriminative frame-based CNN features for the task at-hand in an end-to-end manner or joint learning of parameters of video representation using rank pooled discriminative video representation, and the classifier parameters, (2) a novel temporal encoding method called \emph{hierarchical rank pooling}.

Our proposed method is useful for encoding dynamically evolving frame-based CNN features, and we are able to show significant improvements over other effective temporal encoding methods.

This paper is an extension of our two recent conference papers~\cite{Fernando2016b,Fernando2016a}. In this journal version we provide a broad overview of the action recognition progress and extend the related work section. Here
we unify the learning of discriminative rank pooling and full end-to-end parameter learning using the same bilevel optimization framework. 
Some additional experiments and analysis are also included.
The rest of the paper is organised as follows. 
Related work is discussed in~\secref{sec:related} followed by a brief background to rank pooling and some preliminaries in~\secref{sec:background}. 
We present our discriminative networks in~\secref{sec:discriminative.methods} and discuss how the resulting representation can be used to classify videos. 
In~\secref{sec.learning} we show how all the parameters of the discriminative networks can be learned. 
Then in~\secref{sec.hrp}, we present our hierarchical rank pooling method.
In~\secref{sec:experiments}, we provide extensive experiments evaluating various aspects of our proposed methods.
We conclude the paper in~\secref{sec:conclusion} with a summary of our main contributions and discussion of future directions.

\section{Related Work}
\label{sec:related}
In the literature, temporal information of video sequences is encoded using different techniques. 
Fisher encoding \cite{Perronnin2010} of spatial temporal features is commonly used in prior state-of-the-art works \cite{wang2013action} while Jain \etal \cite{jain2013better} used VLAD encoding \cite{Jegou2010} for action recognition over motion descriptors. 
Temporal max pooling and sum pooling are used with bag-of-features \cite{wang2013dense} as well as CNN features \cite{Ryoo2015}.
Temporal fusion methods such as late fusion or early fusion are used in \cite{karpathy2014large} as a temporal encoding method in the context of CNN architectures. 
In contrast, we rely on principled rank-pooling to encode temporal information inside CNNs and therefore our method is capable capturing dynamics of video sequences.

Temporal information can also be encoded using 3D convolution operators \cite{Ji2013,Tran2015} on fixed size temporal segments. 
However, as recently demonstrated by Tran~\etal \cite{Tran2015}, such approaches rely on very large video collections to learn meaningful 3D-representations. 
This is due to the massive amount of parameters used in 3D convolutions.
Sun~\etal \cite{Sun2015} propose to factorize 3D convolutions into spatial 2D convolutions followed by 1D temporal convolutions to ease the training. 
Moreover, it is not clear how these methods can capture long-term dynamics as 3D convolutions are applied only on short video clips. 
In contrast, our method does not introduce any additional parameters to existing 2D CNN architectures and capable of learning and capturing long term temporal dynamics.


Recently, recurrent neural networks are gaining popularity for sequence encoding, sequence generation and sequence classification \cite{Hochreiter1997,Sutskever2014}. 
Long-short term memory (LSTM) based approaches may use the hidden state of the encoder as a video representation \cite{Srivastava2015}. 
Derivative of the state of the RNN is modelled in differential RNN (dRNN) to capture the dynamics of video sequences \cite{Veeriah2015}.
A CNN feature based LSTM model for action recognition is presented in \cite{Yue-HeiNg2015}. 
Typically, unsupervised recurrent neural networks are trained in a probabilistic manner to maximize the likelihood of generating the next element of the sequence. 
By construction our hierarchical rank pooling method is unsupervised and does not rely on very large number of training samples as in recurrent neural networks as our method does not have any parameters to learn.
Moreover, our hierarchical rank pooling has a clear objective in capturing dynamics of sequences independent of other sequences and has the capacity to capture complex dynamic signals.


Hierarchical methods have also been used in activity recognition \cite{Du2015,Li2016,Song2013}. 
A CRF-based hierarchical sequence summarization method is presented in \cite{Song2013}; a hierarchical recurrent neural network for skeleton based action recognition is presented in \cite{Du2015}; and a hierarchical action proposal based mid-level representation is presented in \cite{Lan2015b}. 
Recently, VLAD for Deep Dynamics (VLAD3), that accounts for different set of video dynamics is presented in \cite{Li2016}. 
It also captures short-term dynamics with deep convolutional neural net-work features, relying on linear dynamic systems (LDS) to model medium-range dynamics. 
To account for long-range inhomogeneous dynamics, a VLAD descriptor is derived for the linear dynamic systems and pooled over the whole video, to arrive at the final VLAD3 representation. 
In contrast to these methods, our method captures different set of mid-level dynamics as well as dynamics of the entire video using rank pooling principle.

Long term temporal dynamics are also modelled using Beta Process Hidden Markov Models (BP-HMM \cite{Fox2009}). 
Using a beta process prior, these approaches discover a set of latent dynamical behaviours that are shared among multiple time series. 
The size of the set and the sharing pattern are both inferred from data. 
Some notable extensions of this approach are used in video analysis and action recognition \cite{Sener2015,Hughes2012}.
Compared to these methods, not only is our framework capable of capturing long term dynamics, it is also capable of capturing dynamics at multiple levels of granularity while being able to learn discriminative dynamics.

Recently, two stream models \cite{Simonyan2014} have gained popularity for action recognition. In these methods, a temporal stream is obtained by using optical flow and spatial stream is obtained by RGB frame data and finally the information is fused \cite{Feichtenhofer2016}. Moreover, trajectory-pooled deep-convolutional descriptor (TDD) also uses two stream network architecture where convolutional feature maps are pooled from the local ConvNet responses over the spatio-temporal tubes centered at the improved trajectories \cite{Wang2015}. Our method presented in this paper is complimentary to these two stream architectures. For example, our hierarchical temporal encoding as well as the end-to-end trainable rank pooled CNN can be applied over both spatial and temporal streams.


Rank pooling is also used for temporal encoding at representation level \cite{Fernando2015,Fernando2016} or at image level leading to dynamic images \cite{Bilen2016}.
However, we are the first to extend rank pooling to a high capacity temporal encoding. Furthermore, we are the first to demonstrate an end-to-end trainable CNN-based rank pool operator.

Our end-to-end learning algorithm introduces a bilevel optimization method for encoding temporal dynamics of video sequences using convolutional neural networks. 
Bilevel optimization \cite{Bard,Gould2016} is a large and active research field derived from the study of non-cooperative games with much work focusing on efficient techniques for solving non-smooth problems \cite{OB15a} or studying replacement of the lower level problem with necessary conditions for optimality \cite{dempe2015}. 
It has recently gained interest in the machine learning community in the context of hyperparameter learning \cite{klatzer2015,Do2007} and in the computer vision community in the context of image denoising \cite{Domke:AISTATS12,kunisch2013}. 
Unlike these works we take a gradient-based approach, which the structure of our problem admits. We also address the problem of encoding and classification of temporal sequences, in particular action and activity recognition in video.

Recently, several end-to-end video classification and action recognition method were introduced in the literature \cite{Ji2013,karpathy2014large,Simonyan2014}. Compare to other end-to-end video representation learning methods our end-to-end learning has two advantages. First, our temporal pooling is based on rank pooling and hence captures the dynamics of long video sequences. Second, it does not introduce any new parameters to existing image classification architectures such as AlexNet \cite{Krizhevsky2012}. Ji~\etal \cite{Ji2013} introduces an end-to-end 3D convolution method that can be only applied for a fixed length videos. Karpathy~\etal \cite{karpathy2014large} used several fusion architectures. Very large Sports-1M dataset was used for training which consist of more than million YouTube videos of sports activities. Unfortunately, authors found that operating on individual video frames, performs similarly to the networks, whose input is a stack of frames. This indicates that the architectures proposed in \cite{karpathy2014large} are not able to learnt spatio-temporal features or capture dynamics of videos. Simonyan~\etal \cite{Simonyan2014} also propose an end-to-end architecture which only operates at frame-level and finally fuse classifier scores per video.


\section{Preliminaries}
\label{sec:background}
In this section we introduce the notation used in this paper and provide background on the \emph{rank pooling} method \cite{Fernando2015,Fernando2016}, which our work extends. 
Given a training dataset of video-label pairs $\D = \{(X^{(0)}, y)\}$, the goal in action classification is to learn both parameters of the classifier and video representation such that the error on the training set is minimized.
Let $X^{(0)} = \langle x_1^{(0)}, \ldots, x_J^{(0)} \rangle$ be the (ordered) sequence of input RGB video frames.\\ 

\noindent
\textbf{Feature extraction function $\psi_{0}()$}:
Let us define a feature extraction function that takes an input frame and returns a fixed-length feature vector by $\psi_{0} : x_t^{(0)} \mapsto \bx_t^{(1)}$. 
This operation transforms a sequence of RGB frames $X^{(0)}$ into a sequence of feature vectors denoted by $X^{(1)} = \langle x_1^{(1)}, \ldots, x_J^{(1)} \rangle$.
Sometimes, to simplify the notation, we denote a sequence of vectors just by $X = \langle \bx_1, \ldots, \bx_J \rangle$.
Each of the elements $\bx_t^{(1)}$ in the sequence $X^{(1)}$ is a vector, i.e., $\bx_t^{(1)} \in \reals^D$.
For example, the vector $\bx_t^{(1)}$ can be the activations from the last fully connected layer of a CNN which is obtained from a RGB video sequence at frame $t$. 
This frame-based feature extractor can be parametrized $\psi_{0}(x_t^{(0)}; \theta)$, where for example, $\theta$ are the parameters of a trainable CNN.\\

\noindent
\textbf{Non-linear operator $\psi()$}:
Let us assume that each video is processed by a feature extractor and then a sequence of vectors is obtained by applying a non-linear transformation. 
Let us denote a point-wise non-linear operator by $\psi()$ and the non-linear transformation is obtained by $\bv_t = \psi(\bx_t)$ or a parametrised non-linear transform is obtained by 
\begin{equation}
\bv_t = \psi(W\bx_t) 
\label{eq:nonlineareq}
\end{equation}
where $W\in\reals^{D \times D}$.
Let us denote the obtained sequence of vectors by $V = \left<\bv_1, \ldots, \bv_J\right>$ where each $\bv_t \in \reals^D$.\\

\noindent
\textbf{Temporal encoding function $\phi()$}:
A compact video representation is needed to classify a variable-length video sequence into one of the activity classes. 
As such, a temporal encoding function that operates over a sequence of vectors is defined by $\phi(V)$, which maps the video sequence $V$ (or sub-sequence thereof) into a fixed-length feature vector, $\bu \in \reals^D$. The goal of temporal encoding is to encapsulate valuable dynamic information in $V$ into a single $D$-dimensional vector $\bu = \phi(V)$. 
In general we can write the temporal encoding function as an optimization problem over a sequence $V$ as
\begin{align}
  \phi(V) &\in \argmin_{\bu} f(V, \bu)
  \label{eqn:general_argmin}
\end{align}
where $f(\cdot, \cdot)$ is some measure of how well the sequence is described by each representation $\bu$ and we seek the best representation.
Standard supervised machine learning classification techniques learned on the set of training videos can then be applied to these $\bu$ vectors.

Typical temporal encoding functions include sufficient statistics calculations or simple pooling operations, such as \texttt{max} or average (\texttt{avg}). 
For example, \texttt{avg.} pooling can be written as the following optimization problem in~\eqnref{eq.avg.opt}.
\begin{align}
  \mathtt{avg}(V)
  &= \argmin_{\bu} \left\{ \frac{1}{2} \sum_{t=1}^{J} \|\bu - \bv_t\|^2 \right\}
  \label{eq.avg.opt}
\end{align}
\\

\noindent
\textbf{Rank pooling}:
The \texttt{max} and \texttt{avg} pooling operators do not capture the dynamic of a video sequence.
More sophisticated, temporal encoders such as the \texttt{rank-pool} operator, attempts to capture temporal dynamics \cite{Fernando2015,Fernando2016}. 
The sequence encoder $\phi(\cdot)$ of rank pooling \cite{Fernando2015,Fernando2016} captures time varying information of the entire sequence using a single linear surrogate function $\zeta$ parametrised by $\bu \in \reals^D$. 
The function $\zeta$ ranks frames of the video $V = \left<\bv_1, \ldots, \bv_J\right>$ based on the chronology based on their feature representation. 
Ideally, the ranking function satisfies the constraint
\begin{align}
  \zeta( \bv_{t_a}) < \zeta(\bv_{t_b}) &\iff t_a < t_b
  \label{eqn:order_constraints}
\end{align}
such that the ranking function should learn to order frames chronologically. 
In the linear case this boils down to finding a parameter vector $\bu$ such that $\zeta(\bv; \bu) = \bu^T\bv$ satisfies \eqnref{eqn:order_constraints}. 
In rank pooling \cite{Fernando2015,Fernando2016} this is done by training a linear ranking machine such as RankSVM \cite{JoachimsKDD2006} on $V$. The learned parameters of RankSVM, \ie, $\bu$, are then used as the temporal encoding of the video.
Since the ranking function encapsulates ordering information and the parameters lie in the same feature space, the ranking function captures the \emph{evolution of information} in the sequence $V$ \cite{Fernando2015,Fernando2016}.

Rank pooling can be viewed as a function that estimates the parameters $\bu$ in a point-wise manner such that it maps feature vectors $v_t$ to time $t$. 
Such a mapping clearly satisfies the order constraints of \eqnref{eqn:order_constraints}. 
The idea of rank pooling is to parameterize $\zeta$ and then find the parameters $\bu^\star$ that best represents the sequence $V$. 
Due to availability of fast implementations, we use Support Vector Regression (SVR) \cite{Liu2009} to solve this problem. 
Given a sequence of length $J$, the SVR parameters are given by
\begin{align}
 \bu^\star \!&\in \argmin_{\bu} \!\left\{ \!\frac{1}{2}\|\bu\|^{2} + \frac{C}{2} \sum_{t=1}^J \Big[ |t - \bu^\top \bv_t| - \epsilon\Big]_{\geq 0}^2 \!\right\}
 \label{eq.svr}
\end{align}
where $\left[\cdot\right]_{\geq 0} = \max\{\cdot, 0\}$ projects onto the positive reals.

The advantage of stability and robustness in modelling dynamics is discussed in \cite{Fernando2016}. 
As the SVR objective has some theoretical guarantees on the generalization and stability \cite{BousquetJMLR2002} the obtained temporal representation $\bu^\star$ is robust to small perturbed versions of the input. Therefore, the above SVR objective is advantageous for modelling dynamics.
We use the parameter $\bu^\star$, returned by SVR, as the temporal encoding vector of the video sequence.


\subsection{overview}
One of the limitations of rank pooling method presented in \cite{Fernando2015,Fernando2016} is that obtained temporal representation $\bu$ is not discriminative as the classifier and the underlying frame representation is obtained independently. 
In this work we extend the work of Fernando~\etal \cite{Fernando2015,Fernando2016}. 
First, we show a learning framework for discriminative temporal encoding using rank pooling in section~\ref{sec:discriminative.methods}.
Given a collection of labelled videos, we show how to learn frame representation, temporal representation for the video and the classifier jointly. 
In this case, the temporal representation is obtained by \texttt{rank-pool} operator. 
We also learn a discriminative rank pooling operator when a set of labelled sequences of vectors are provided as the input. 
In this case, we learn the classifier parameters and the discriminative temporal representation jointly.
Parameter learning of these discriminative models is explained in section~\ref{sec.learning}.
Second, we show hierarchical rank-pooling, a new hierarchical temporal encoding scheme which extends the \texttt{rank-pool} operator in section~\ref{sec:hrp}.
To learn discriminative hierarchical representation, one can stack discriminative rank pooling network over the hierarchical rank pooling network.
In experiments, we demonstrate how to combine hierarchical rank pooling with discriminative learning framework to obtain good results for action recognition (section~\ref{sec.exp.disk.rp}).


\section{Discriminative video representations with rank-pooling networks}
\label{sec:discriminative.methods}

In this section, we introduce our proposed trainable rank pooling network based video representation framework.
We consider two scenarios to learn discriminative video representations using \texttt{rank-pool} operator.
In both cases, the temporal encoding of frame level feature vectors is obtained with rank pooling.
\begin{enumerate}
 \item In the first scenario, the input to our algorithm is a set of labelled row RGB videos $\D = \{(X^{(0)}, y)\}$. 
 Then our aim is to learn parametrized feature extractor (a CNN \cite{Krizhevsky2012} feature extractor which is denote by $\bv_t = \psi_{0}(x_t^{(0)}; \theta)$), the temporal video representation ($\bu$) and the action classifiers jointly. In this case $\theta$  is the set of parameters in a trainable CNN. 
 \item In the second scenario, input to our algorithm is a set of labelled sequences of vectors obtained from video sequences. 
 We aim to learn a parameterized non-linear operation denoted by Equation~\eqref{eq:nonlineareq} and the classifier parameters jointly.  
 The $W$ matrix is shared across all sequences from all classes.
\end{enumerate}

Next, we provide more details about these two models.
First, we discuss our end-to-end video representation and classification model in~\secref{sec.endtoend}. 
Then in~\secref{sec.discriminative}, we introduce the discriminative \texttt{rank-pool} operator that operates over a sequences of vectors.

\subsection{End-to-end trainable rank pooled CNN}
\label{sec.endtoend}

In the first scenario, the input to our framework is a sequence of raw RGB videos with action category labels $\D = \{(X^{(0)}, y)\}$.
We assume that each video frame in the input sequence is encoded by a CNN network \cite{Krizhevsky2012} $\psi_{0}(\cdot; \btheta)$ which is parameterized by $\theta$ and that the resulting sequence of features $V=\left< \ldots, \bv_t, \ldots \right>$ is encoded using rank pooling (the temporal encoder $\phi$) by solving the objective function in Equation~\eqref{eq.svr}. 
The model we propose can be summarized by the following network equation:
\begin{align}
X^{(0)} = \left< x_t^{(0)} \right>
\overset{\psi_{0}(\cdot; \btheta)}{\longmapsto}
\left< \bv_t \right>
\overset{\phi}{\longmapsto}
\bu
\overset{h_{\bbeta}}{\longmapsto}
\hat{y}
\end{align}
where the feed-forward pass of the network go from a video sequence $X^{(0)}$ to predicted label $\hat{y}$.
The final layer is our prediction function (a soft-max classifier) $h_{\bbeta}$ parameterized by $\bbeta$. 
Therefore, the probability of a label $y$ given the input sequence $X^{(0)}$ can be written as
\begin{align}
P(y \mid X^{(0)}; \bbeta, \btheta) &= \frac{\exp(\beta_y^\top \bu)}{\sum_{c} \exp(\beta_{c}^\top \bu)}
\label{eqn:softmax}
\end{align}
where we have used $\bu$ to denote the final video encoding. 
Importantly, $\bu$ is a function of both the input video sequence $X^{(0)}$ and the network parameters $\btheta$. 
Here the predictor function $h_{\bbeta}(\bu)$ takes the highest probability (most likely) $y$ over the discrete set of labels and $\bbeta = \{\beta_y\}$ are the learned parameters of the model.

The detailed network architecture is shown in Figure~\ref{fig:cnnnet}. 
We use a CNN architecture similar to CaffeNet \cite{Jia2014} with the addition of a temporal pooling layer. 
In our experiments we use the final activation layers of the CNN as the frame level features and then apply the temporal pooling (\texttt{rank-pool} operator) as shown in Figure~\ref{fig:cnnnet}. 
\begin{figure}
 \centering 
 \includegraphics[width=\linewidth]{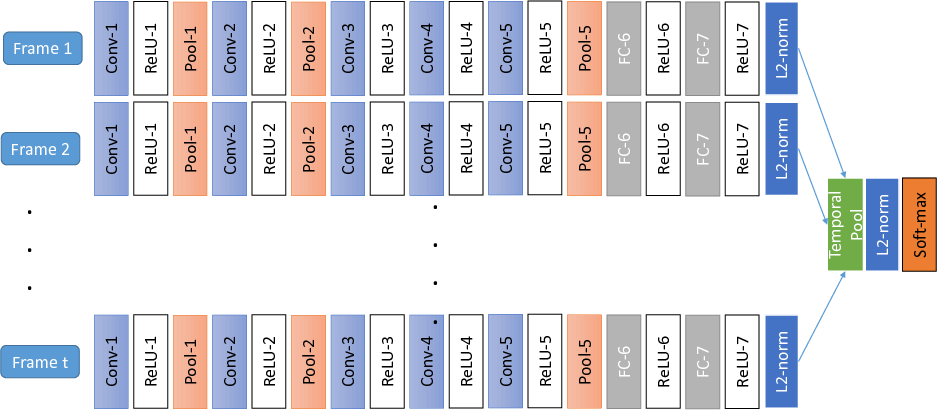}
 \caption{The CNN network architecture used for learning end-to-end temporal representations of videos. Network takes a sequence of frames from a video as inputs and feed forward till the end of the temporal pooling layer. At the temporal polling layer, the sequence of vectors are encoded by rank-pooling operator to produce fixed length video representation. This fixed length vector is feed to the next layer in the network. Note that this network does not introduce any new parameters to network architectures. During back-propagation, the gradients are feed backwards through the rank-pooling operator to the rest of the CNN network.}
 \label{fig:cnnnet} 
\end{figure}
During training, our objective is to learn the parameters $\bbeta$ and $\btheta$. During inference we fix $\btheta$ and $\bbeta$ to their learned values; $\btheta$ is used to obtain the frame representation of the video that is used to obtain $\bu$ via temporal encoding and which is then classified (using parameters $\btheta$) into an estimated action class for the video.

\subsection{Discriminative rank pooling}
\label{sec.discriminative}
In this section, we discuss the second model where the input to the feature extractor is a sequence of vectors instead of sequence of RGB frames.
We present a method to learn dynamics of any vector sequence in a discriminative manner using \texttt{rank-pool} operator as the temporal encoder.
In this instance, the parameterized non-linear operation as in Equation~\eqref{eq:nonlineareq} is applied over the feature vectors of the sequence $X = \left<\bx_1, \ldots,\bx_t, \ldots \bx_J\right>$.
The function $\psi$ is a non-linear feature function such as ReLU \cite{Krizhevsky2012}.
The discriminative rank pooling network can be summarized as follows:

\begin{align}
X = \left< \bx_t \right>
\overset{\psi(\cdot; W)}{\longmapsto}
\left< \bv_t \right>
\overset{\phi}{\longmapsto}
\bu
\overset{h_{\bbeta}}{\longmapsto}
\hat{y}
\end{align}
where $h_{\bbeta}$ is the soft-max classifier parameterized by $\bbeta$. 
Similar to ~\secref{sec.endtoend}, our aim is to jointly learn the non-linear transformation parameter $W$ of $\psi(\cdot; W)$ along with the classifier parameters denoted by $\bbeta$.


\section{Learning the parameters of rank pooling networks}
\label{sec.learning}
Now we have presented our two video representation models in the previous section, we discuss how to learn the parameters in this section.
First, we formulate the overall learning problem in~\secref{sec.opt} and then we show how to learn the parameters with stochastic gradient descent in~\secref{sec.sgd}.
Then we compute the gradient function of our two models in~\secref{sec.sgd.cnn} and~\secref{sec.sgd.dis} respectively.
Finally, we discuss some optimization difficulties and solutions in~\secref{sec.opt.diff}.

\subsection{Optimization problem}
\label{sec.opt}
The learning problem can be described as follows. 
Given a training dataset of video-label pairs $\D = \{(X^{(0)}, y)\}$ (or $\D = \{(X, y)\}$), our goal is to learn both parameters of the classifier and video representation $\theta$ (or $W$)  such that the error on the training set is minimized. 
Let $\Delta(\cdot, \cdot)$ be a loss function. 
For example, when using the soft-max classifier a typical choice would be the cross-entropy loss
\begin{align}
  \Delta(y, h_{\bbeta}(\bu)) &= -\log P(y \mid X^{(0)})
\end{align}
where $P(\cdot \mid \cdot)$ is defined by Equation~\eqref{eqn:softmax}.

We jointly estimate the parameters of the feature extractors ($\theta$ or $W$) and prediction function ($\bbeta$) by minimizing the regularized empirical risk. 
Formally, our learning problem for end-to-end trainable rank pooled CNN is
\begin{align}
  \begin{array}{ll}
    \text{minimize}_{\btheta, \bbeta} & \sum_{(X^{(0)}, y) \in \D} \Delta\!\left(y,\, h_{\bbeta}(\bu_X)\right) + R(\btheta, \bbeta)
    \\
    \text{subject to} & \bu_X \in \argmin_{\bu} f(V, \bu; \btheta )
  \end{array}
  \label{eqn:learning}
\end{align}
where $R(\cdot, \cdot, \cdot)$ is some regularization function, typically the $\ell_2$-norm of the parameters, and the function $f(\cdot)$ encapsulates the temporal encoding of the video sequence using rank pooling temporal encoder $\phi$ by solving~\eqref{eq.svr}. 
The vector $\bu_X$ then represents the output of the rank pooling operator.
It should be noted that the learning problem for discriminative rank pooling of~\secref{sec.discriminative} is similar to the Equation~\eqref{eqn:learning}.

\eqnref{eqn:learning} is an instance of a bilevel optimization problem, which have recently been explored in the context of support vector machine (SVM) hyper-parameter learning \cite{klatzer2015} but whose history goes back to the 1950s \cite{Bard}. Here an upper level problem is solved subject to constraints enforced by a lower level problem. A number of solution methods have been proposed for bilevel optimization problems. 
Given our interest in learning video representations, which is large-scale, gradient-based techniques are most appropriate to learn the parameters.


\subsection{Learning with stochastic gradient descent}
\label{sec.sgd}
We are now left with the task of tuning the parameters $\theta$ or $W$ to learn a discriminative video representation in order to improve the action recognition performance.
One such approach is to learn the classifier parameters and feature encoding parameters jointly via stochastic gradient descent (SGD). 
However, this requires back propagation of gradients through the network. 
When the temporal encoding function $\phi$ can be evaluated in closed-form (e.g., \texttt{max} or \texttt{avg} pooling) to obtain the temporal encoding vector $\bu_X$, we can substitute the constraints in \eqnref{eqn:learning} directly into the objective and use \mbox{(sub-)gradient} descent to solve for (locally or globally) optimal parameters. 
However, when rank pooling is used for temporal encoding the situation is not as simple. 
Recall that the rank pooling operator is itself an optimization problem, which takes an arbitrary long sequence of feature vectors and returns a fixed-length vector that preserves temporal information. 
In this instance, the gradient of an $\argmin$ function is required. 
Fortunately, when the lower level objective is twice differentiable we can compute the gradient of the $\argmin$ function as other authors have also observed \cite{OB15a,Domke2012,Do2007}.
We repeat the key result here for completeness.

\begin{lemma} \cite{Samuel:CVPR09}
Let $f: \reals \times \reals^n \rightarrow \reals$ be a continuous
function with first and second derivatives. 
\break
Let $\bg(x) = \argmin_{\by \in \reals^n}
f(x, \by)$. Then
\begin{align*}
  \bg'(x) &= - f_{YY}(x, \bg(x))^{-1} f_{XY}(x, \bg(x)).
  \label{eqn:vector_argmin}
\end{align*}
where $f_{YY} \doteq \nabla^{2}_{\by\by} f(x, \by) \in \reals^{n \times n}$ and 
$f_{XY} \doteq \frac{\partial}{\partial x} \nabla_{\by} f(x, \by) \in \reals^{n}$.
\label{lem:vector_argmin}
\end{lemma}

\begin{proof}
We have:
\begin{align}
f_{Y}(x, \bg(x)) \doteq {\nabla_{Y} f(x, \by)}|_{\by=\bg(x)} &= 0 \\
\frac{d}{dx} f_{Y}(x, \bg(x)) &= 0 \\
\therefore f_{XY}(x, \bg(x)) + f_{YY}(x, \bg(x)) \bg'(x) &= 0 \\
\frac{d}{dx}\bg(x) = - {f_{YY}(x, \bg(x))}^{-1} f_{XY}(x, \bg(x))
\end{align}
\qed
\end{proof}

Interestingly, replacing $\argmin$ with $\argmax$ in the above lemma yields the same gradient, which follows from the proof that only requires that $\bg(x)$ be a stationary point. So the result holds for both $\argmin$ and $\argmax$ optimization problems.

Using Lemma~\ref{lem:vector_argmin} we can compute the gradient of the rank pooling temporal encoding function with respect to a parameterized representation of the feature vectors. We only consider the case of a single scalar parameter $\theta$. The extension to a vector of parameters can be done elementwise.

\begin{COROLLARY}
Let $\theta \in \reals$ be a parameter and let $\langle \bv_t \rangle$ be a sequence where the $\bv_t$ are functions of $\theta$. Define $f(\theta, \bu)$ to be the objective of the rank pooling optimization problem~\eqnref{eq.svr}. That is,
\begin{align*}
  f(\theta, \bu) &= \frac{1}{2}\|\bu\|^{2} + \frac{C}{2} \sum_{t=1}^J \Big[ |t - \bu^\top \bv_t| - \epsilon\Big]_{\geq 0}^2.
\end{align*}
And let $\bu^\star = \argmin_{\bu} f(\theta, \bu)$. Then 
\begin{align*}
  \frac{d \bu^\star}{d\theta} &= 
  \left(I + C \sum_{e_t \neq 0} \! \bv_t \bv_t^\top \right)^{\!-1}
  \! \left(C \sum_{e_t \neq 0} \! e_t \frac{d v_t}{d\theta}
  - \bu^{\star\top} \frac{d v_t}{d\theta} \bv_t\right)
\end{align*}
where
\begin{align*}
e_t &\doteq \begin{cases}
  \bu^{\star\top} \bv_t - t - \epsilon, & \text{if $\bu^{\star\top} \bv_t - t \geq \epsilon$}
  \\
  \bu^{\star\top} \bv_t - t + \epsilon, & \text{if $t - \bu^{\star\top} \bv_t \geq \epsilon$} 
  \\
  0, & \text{otherwise.}
  \end{cases}
\end{align*}
\label{lem:rank_pool_gradient}
\end{COROLLARY}

\begin{proof}
Follows from Lemma~\ref{lem:vector_argmin} with
\begin{align}
  f_{U}(\theta, \bu^\star) &= \bu^\star - C \sum_{t=1}^{J} e_t \bv_t
  \\
  f_{UU}(\theta, \bu^\star) &= I + C \sum_{e_t \neq 0} \bv_t \bv_t^\top
  \\
  f_{\theta U}(\theta, \bu^\star) &= - C \sum_{e_t \neq 0} e_t \frac{d \bv_t}{d \theta} + 
  \bu^{\star\top} \frac{d \bv_t}{d \theta} \bv_t
\end{align}
\qed
\end{proof}

In the subsections below we discuss the specifics of learning the parameters of our two parametric discriminative models ($\btheta$ and $W$).

\subsection{Learning the parameter of end-to-end trainable rank pooled CNN}
\label{sec.sgd.cnn}

Now we present how to learn the parameters of the CNN ($\btheta$) and the classifier parameters.
Consider again the learning problem defined in \eqnref{eqn:learning}. 
The derivative with respect to $\bbeta$, which only appears in the upper-level problem, is straightforward and well known.
Using the result of Corollary~\ref{lem:rank_pool_gradient}, we compute $\frac{d \bu^{(i)}}{d\theta}$ for each training example and hence the gradient of the objective via the chain rule.%
We then use stochastic gradient descent (SGD) to learn all parameters jointly.

Consider a single scalar weight update in the CNN. 
Then, again using Lemma~\ref{lem:rank_pool_gradient} we have
\begin{multline}
\frac{d \bu^{(i)}}{d\theta} =
  \left(I + C \sum_{e_t \neq 0} \! \bv_t \bv_t^\top \right)^{\!-1} 
  \\
  \left(C \sum_{e_t \neq 0} \! e_t \psi'_0(\bx_t; \theta) - \bu^\top \psi'_0(\bx_t; \theta)\bv_t\right)
\label{eqn:gradient}
\end{multline}
Here $\psi_{0}^\prime(\bx_t;\theta)$ is the derivative of the element feature function. 
In the context of CNN-based features for encoding video frames the derivative can be computed by back-propagation through the network.
Note that the \texttt{rank-pool} objective function is convex and allows us to solve it efficiently.
However, it does include a set of non-differentiable points but we did not find this to cause any practical problems during optimization.

\subsection{Learning the parameter $W$ of discriminative rank pooling}
\label{sec.sgd.dis}

Recall, in discriminative  rank pooling network model, the sequence of vectors $X$ is processed by optimizing \eqnref{eq.svr} to get $\bu$, where  $\bv_t = \psi(W \bx_t)$. Objective is to learn the classifier parameters and the parameter $W$ jointly.
The derivative with respect to classifier parameter $\bbeta$, which only appears in the upper-level problem, is straightforward and well known.
However, the partial derivative w.r.t. $W$ is more challenging since $\bu$ is a complicated function of $W$ defined by \eqnref{eq.svr}, which involves solving an argmin optimization problem as before. 
Thus we have to differentiate \emph{through} the argmin function of the rank pooling problem using Lemma~\ref{lem:rank_pool_gradient}.

Recall, we have $\bv_t = \psi(W \bx_t)$ where $\psi(\cdot)$ acts elementwise. From Lemma~\ref{lem:rank_pool_gradient} we have for parameter $W_{ij}$
\begin{align}
  \frac{\partial \bu}{\partial W_{ij}} &= 
  \left(I + C \sum_{e_t \neq 0} \! \bv_t \bv_t^\top \right)^{\!-1}
  \! \left(C \sum_{e_t \neq 0} \! e_t \frac{\partial v_t}{\partial W_{ij}}
  - \bu^{\top} \frac{\partial v_t}{\partial W_{ij}} \bv_t\right)
  \label{eqn:du_dW_full}
\end{align}
where the $k$-th element of $\frac{\partial v_t}{\partial W_{ij}}$ is
\begin{align}
  \left(\frac{\partial \bv_t}{\partial W_{ij}}\right)_{\!k} &= \begin{cases}
    \psi'(W \bx_t)_{[k]} \bx_{t[j]}, & \text{if $k = i$} \\
    0, & \text{otherwise.}
    \end{cases}
\end{align}
Here the subscript $[i]$ denotes the $i$-th element of the associated vector.

\subsection{Optimization difficulties}
\label{sec.opt.diff}
One of the main difficulties for learning the parameters of high-dimensional temporal encoding functions (such as those based on CNN features) is that the gradient update in \eqnref{eqn:gradient} requires the inversion of the Hessian matrix $f_{UU} = (I + C \sum_{e_t \neq 0} \bv_t \bv_t^\top)$. 
One solution is to use a diagonal approximation of the Hessian, which is trivial to invert.
For instance let us compute the gradient of discriminative rank pooling model using the diagonal approximation.
Considering the derivative of the $k$-th element of $\bu$ and approximating the inverse of the first term in \eqnref{eqn:du_dW_full} by its diagonal, we have
\begin{multline}
  \frac{\partial \bu_k}{\partial W_{ij}} =
  \left(\frac{1}{1 + C \sum_{e_t \neq 0} \bv_{t[k]}^2} \right) \times \\
  \left(C \sum_{e_t \neq 0} \! \left(\ind{i = k} e_t - \bu_i \bv_{t[k]}\right) \psi'_i(W \bx_t) \bx_{t[j]} \right)
\label{eqn:dukdWij}
\end{multline}

Now we have by the chain rule,
\begin{align}
\frac{\partial}{\partial W_{ij}} \log P(y \mid X) 
&=
\big(\nabla_{\bu} \log P(y \mid X)\big)^\top \frac{\partial \bu}{\partial W_{ij}}
\\
&= \left(\beta_y - \sum_{c} P(c \mid X) \beta_c\right)^\top \! \frac{\partial \bu}{\partial W_{ij}}
\label{eqn:dPdWij}
\end{align}

Let $\ones$ be the all-ones vector, let $K_t = \frac{\partial {\bv}_t }{\partial W}$ where $(K_t)_{[ij]} \doteq \frac{\partial \bv_{t[i]}}{\partial W_{ij}}$ and let $\hat{\beta}$ denote $\nabla_{\bu} \log P(y \mid X)$ scaled by the inverse diagonal hessian, i.e.,
\begin{align}
\hat{\beta}_{[i]} &= \frac{\beta_{y[i]} - \sum_{c} P(c \mid X) \beta_{c[i]}}{1 + C \sum_{e_t \neq 0} \bv_{t[i]}^2}
\end{align}
Then we can write \eqnref{eqn:dPdWij} more compactly as
\begin{align}
  \frac{\partial}{\partial W_{ij}} \log P(y \mid X) &=
  \sum_{e_t \neq 0} \left( \hat{\beta}_{[i]} e_t - \bu_i \hat{\beta}^\top \bv_t\right) K_{t[ij]}
\end{align}
and the (matrix) gradient with respect to all parameters as
\begin{align}
  \nabla_{W} \log P(y \mid X) &=
  C \sum_{e_t \neq 0} e_t K_t \odot (\hat{\beta} \ones^T) - 
    s_t K_t \odot (\bu \ones^T) 
\label{eq.der.completeW}  
\end{align}
where $\odot$ is the Hadamard product and $s_t \doteq \hat{\beta}^\top \bv_t$.

An alternative, for temporal encoding functions with certain structure like ours, namely where the hessian can be expressed as a diagonal plus the sum of rank-one matrices, the inverse can be computed efficiently using the Sherman-Morrison formula \cite{Golub},
\begin{lemma} \cite{Golub}
Let $H = I + \sum_{i=1}^n u_i v_i^\top \in \reals^{p \times p}$ be invertible. Define $H_0 = I$ and $H_m = H_{m-1} + u_m v_m^\top$ for $m = 1, \ldots, n$. Then
\begin{align}
  H_m^{-1} &= H_{m-1}^{-1} - \frac{H_{m-1}^{-1} u_m v_m^\top H_{m-1}^{-1}}{1 + v_m^\top H_{m-1}^{-1} u_m}
  \label{eqn:inv_update}
\end{align}
whenever $v_m^\top H_{m-1}^{-1} u_m \neq -1$.
\label{lem:inverse}
\end{lemma}

\begin{proof}
Follows from repeated application of the Sherman-Morrison formula.
\end{proof}

Since each update in \eqnref{eqn:inv_update} can be performed in $O(p^2)$ the inverse of $H$ can be computed in $O(np^2)$, which is acceptable for many applications. Our experiments include results onbtained by both the diagonal approximation and full inverse.

\section{HRP: Hierarchical rank pooling}
\label{sec:hrp}
In this section we present our \emph{hierarchical rank pooling} (HRP) network for video classification. 
HRP is an unsupervised temporal encoding network which allows us to obtain high capacity temporal encoding.

Even with a rich feature representation of each frame in a video sequence, such as derived from a deep convolutional neural network (CNN) model \cite{Krizhevsky2012}, the shallow rank pooling method \cite{Fernando2015,Fernando2016} may not be able to adequately model the dynamics of complex activities over long sequences. 
As such, we propose a more powerful yet simple scheme for encoding the dynamics of rich features of complex video sequences. 
Motivated by the success of hierarchical encoding of deep neural networks \cite{Krizhevsky2012,Girshick2014}, we extend rank pooling operator to encode dynamics of a sequence at multiple levels in a hierarchical manner. 
Moreover, at each stage, we apply a non-linear feature transformation to capture complex dynamical behaviour. We call this method the \emph{hierarchical rank pooling}.

Our main idea is to perform rank pooling on sub-sequences of the video. 
Each invocation of rank pooling provides a fixed-length feature vector that describes the sub-sequence. 
Importantly, the feature vectors capture the evolution of frames within each sub-sequence. 
By construction, the sub-sequences themselves are ordered. 
As such, we can apply rank pooling over the generated sequence of feature vectors to obtain a higher-level representation. 
This process is repeated to obtain  dynamic representations at multiple levels for a given video sequence until we obtain a final encoding. 
To make this hierarchical encoding even more powerful, we apply a point-wise non-linear operation on the input to the rank pooling function. 
An illustration of the approach is shown in \figref{fig:hpooling}.

\label{sec.hrp}
\begin{figure}
 \centering
 \includegraphics[width=\linewidth]{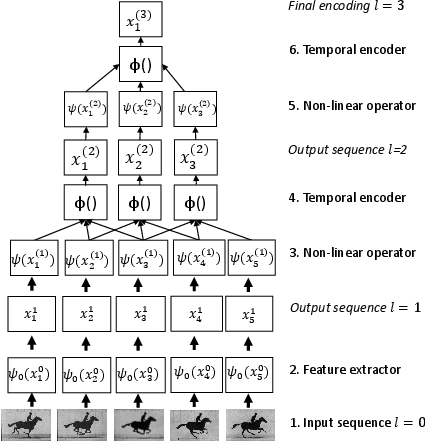}
 \caption{Two layer network of hierarchical rank pooling with window size three ($\M_\ell = 3$) and stride one ($S_\ell = 1$). \label{fig:hpooling}}
 \vspace{-2mm}
\end{figure}


We assume CNN features are extracted from a fixed CNN.
Using a slight change in the notation we denote this by $\bx_t^{(1)} = \psi_{0}(x_t^{(0)};\theta)$ where the $\theta$ is fixed.
In unsupervised hierarchical rank pooling method, we extract feature vectors from each of the frame resulting a sequence of vectors denoted by  
\begin{align}
X^{(1)} &= \langle \psi_{0}(x_1^{(0)}), \ldots, \psi_{0}(\bx_J^{(0))} \rangle.  
\end{align}
We then apply a non-linear transformation $\psi$ to each feature vector to obtain a transformed sequence 
\begin{equation}
\widetilde{X}^{(1)} = \langle \psi(\bx_1^{(1)}), \ldots, \psi(\bx_J^{(1)}) \rangle. 
\end{equation}
Next, applying rank pooling-based temporal encoding $\phi$ to sub-sequences of $\widetilde{X}^{(1)}$, we obtain a new sequence $X^{(2)}$ of feature vectors describing each video sub-sequence. 
The process of going from $X^{(1)}$ to $X^{(2)}$ constitutes the first layer of the temporal hierarchy. 
We now extend the process through additional rank pooling layers, which we formalize by the following definition. Indeed, in our implementation the temporal encoding function $\phi$ is \texttt{rank-pool} operator.

\begin{defn}[{\bf Rank Pooling Layer}]

Let $X^{(\ell)} = \langle\bx^{(\ell)}_1, \ldots, \bx^{(\ell)}_{J_{\ell}}\rangle$ be a sequence of $J_{\ell}$ feature vectors. Let $M_\ell$ be the window size and $S_\ell$ be a stride. For $t \in \{1, S_\ell + 1, 2S_\ell + 1, \ldots\}$ define transformed sub-sequences $\widetilde{X}^{(\ell)}_t = \langle\psi_{\ell}(\bx^{(\ell)}_t), \ldots, \psi_{\ell}(\bx^{(\ell)}_{t + M_\ell - 1})\rangle$, where $\psi_{\ell}(\cdot)$ is a point-wise non-linear transformation. Then the output of the $\ell$-th \emph{rank pooling layer} is a sequence $X^{(\ell + 1)} = \langle\ldots, \bx^{(\ell + 1)}_t, \ldots\rangle$ where $\bx^{(\ell + 1)}_t = \phi(\widetilde{X}^{(\ell)}_t)$ is a temporal encoding of the transformed sub-sequence $\widetilde{X}^{(\ell)}_t$ obtained by \texttt{rank-pool} operator.
\end{defn}

Each successive layer in our rank pooling hierarchy captures the dynamics of the previous layer. 
The entire hierarchy can be viewed as applying a stack of non-linear ranking functions on the input video sequence and shares some conceptual similarities with deep neural networks. 
A simple illustration of a two-layer hierarchical rank pooling network is shown in \figref{fig:hpooling}.
By varying the stride and window size for each layer, we control the depth of the rank pooling hierarchy. There is no technical reason to limit the number of layers. 

To obtain the final vector representation $\bx^{(L + 1)}$, we construct the sequence for the final layer $X^{(L)}$, and encode the whole sequence $X^{(L)}$ with \texttt{rank-pool} operator $\phi(\widetilde{X}^{(L)})$. 
In other words, the last layer in our hierarchy produces a single temporal encoding of last output sequence $\widetilde{X}^{(L)}$ using \texttt{rank-pool} operator. 
We use this final feature vector $\bx^{(L + 1)}$ of the video as its representation, which is then classified by a SVM classifier.

\subsection{Capturing non-linear dynamics with non-linear feature transformations}
\label{sec:method:non-linear-maps}

Usually, video sequence data contains complex dynamic information that cannot be captured simply using linear methods such as linear SVR. 
We believe that the dynamics captured by standard SVR objective reflects only linear dynamics as the SVR function is  linear.
To obtain non-linear dynamics, one option is to use non-linear feature maps and transform the input features by a non-linear operation.
Here we transform the input vectors $\bx_t$ by a non-linear operation $\psi(\bx_t)$ before applying SVR based rank pooling (Equation~\eqref{eq.svr}).
In the literature, Signed Square Root (SSR) and Chi-square feature mappings are used to obtain good results. 
Neural networks employ sigmoid and hyperbolic tangent functions to model non-linearity. 
The advantage of SSR is exploited by Fisher vector-based object recognition as well as in activity recognition \cite{Fernando2015,wang2013action}. 
When CNN features are used to represent frames, we suggest to consider positive activations separately from the negative activations. 
Typically the rectification applied in CNN architectures keeps only the positive activations, \ie, $\psi(\bx) = \max\{\zeros,\bx\}$. 
However, we argue that negative activations may also contain some useful information and should be considered.
Therefore, we propose to use the following non-linear function on the activations of fully connected layers of the CNN architecture. 
We call this operation the \emph{sign expansion root (\textbf{SER}).
\begin{align}
\psi(x) &= \left( \sqrt{\max\{0, x\}}, \sqrt{\max \{0, -x\}} \right)
 \label{eq.signexpansion}
\end{align}
This operation doubles the size of the features space allowing us to capture important non-linear information, one for positives and the other for negatives. 
The square-root operation takes care of projecting features to a some unknown non-linear feature space.}

So far in this \secref{sec:hrp}, we have described how to represent a video by a fixed-length descriptor using hierarchical rank pooling in an unsupervised manner. These descriptors can be used to learn an SVM classifier for activity recognition. The forward pass algorithm for hierarchical rank pooling is shown in Algorithm~\ref{alg:fwd_pass}.

\RestyleAlgo{boxruled}
\begin{algorithm}[h!]
  \begin{algorithmic}[1]
    \STATE{extract CNN features, $X^{(1)} = \langle\bx^{(1)}_1, \bx^{(1)}_2, \ldots, \bx^{(1)}_J\rangle$}
    \FOR{each rank pooling layer, $\ell = 1 : L-1$}
    \STATE{generate transformed sub-sequences $\widetilde{X}^{(\ell)}_t$ as $\langle \psi(\bx^{(\ell)}_t) \rangle$}
    \STATE{rank pool each sub-sequence, $\bx^{(\ell + 1)}_t = \phi(\widetilde{X}^{(\ell)}_t)$}
    \STATE{construct ${X}^{(\ell+1)}$ as $\langle \ldots, \bx^{(\ell + 1)}_t, \ldots \rangle$}
    \ENDFOR
    \STATE{get video representation as $\bx^{(L+1)} = \phi(\widetilde{X}^{(L)})$}
  \end{algorithmic}
  \caption{\label{alg:fwd_pass} Hierarchical Rank Pooling Forward Pass.}
\end{algorithm}

\section{Experiments}
\label{sec:experiments}

We evaluate proposed methods using four activity and action recognition datasets. 
We follow exactly the same experimental settings per dataset, using the same training and test splits as described in the literature. Now we give some details of these datasets (also see~\figref{fig.datasets}).

\noindent\textit{HMDB51 dataset \cite{Kuehne2011}} is a generic action classification dataset consists of 6,766 video clips divided into 51 action classes. Videos and actions of this dataset are challenging due to various kinds of camera motions, viewpoints, video quality and occlusions. Following the literature, we use a one-vs-all multi-class classification strategy and report the mean classification accuracy over three standard splits provided by Kuehne~\etal \cite{Kuehne2011}.\\ 

\noindent\textit{Hollywood2 dataset} is created by Laptev~\etal \cite{Laptev2008} using 69 different Hollywood movies that include 12 human action classes. It contains 1,707 video clips in which 823 clips are dedicated for training and 884 clips for testing. The performance is measured by average precision. The mean average precision (mAP) is reported over all classes, as in \cite{Laptev2008}.\\ 

\noindent\textit{UCF101 dataset \cite{soomro2012ucf101}} is an action recognition dataset of realistic action videos, collected from YouTube, consists of 101 action categories. It has 13,320 videos from 101 diverse action categories. The videos of this dataset is challenging which contains large variations in camera motion, object appearance and pose, object scale, viewpoint, cluttered background and illumination conditions. It is one of the most challenging data set to date. It consist of three splits, in which we report the classification performance over all three splits as done in the literature.\\

\noindent\textit{UCF-sports dataset \cite{Rodriguez2008}} consists of a set of short video clips depicting actions collected from various sports. The clips were typically sourced from footage on broadcast television channels such as the BBC and ESPN. The video collection represents a natural pool of actions featured in a wide range of scenes and viewpoints. The dataset includes a total of 150 sequences of resolution $720 \times 480$ pixels. Classification performance is measured using mean per-class accuracy. We use provided train-test splits for training and testing.\\

\begin{figure}[t!]
\begin{center}
\subfloat[HMDB51]{\includegraphics[width=.9\columnwidth]{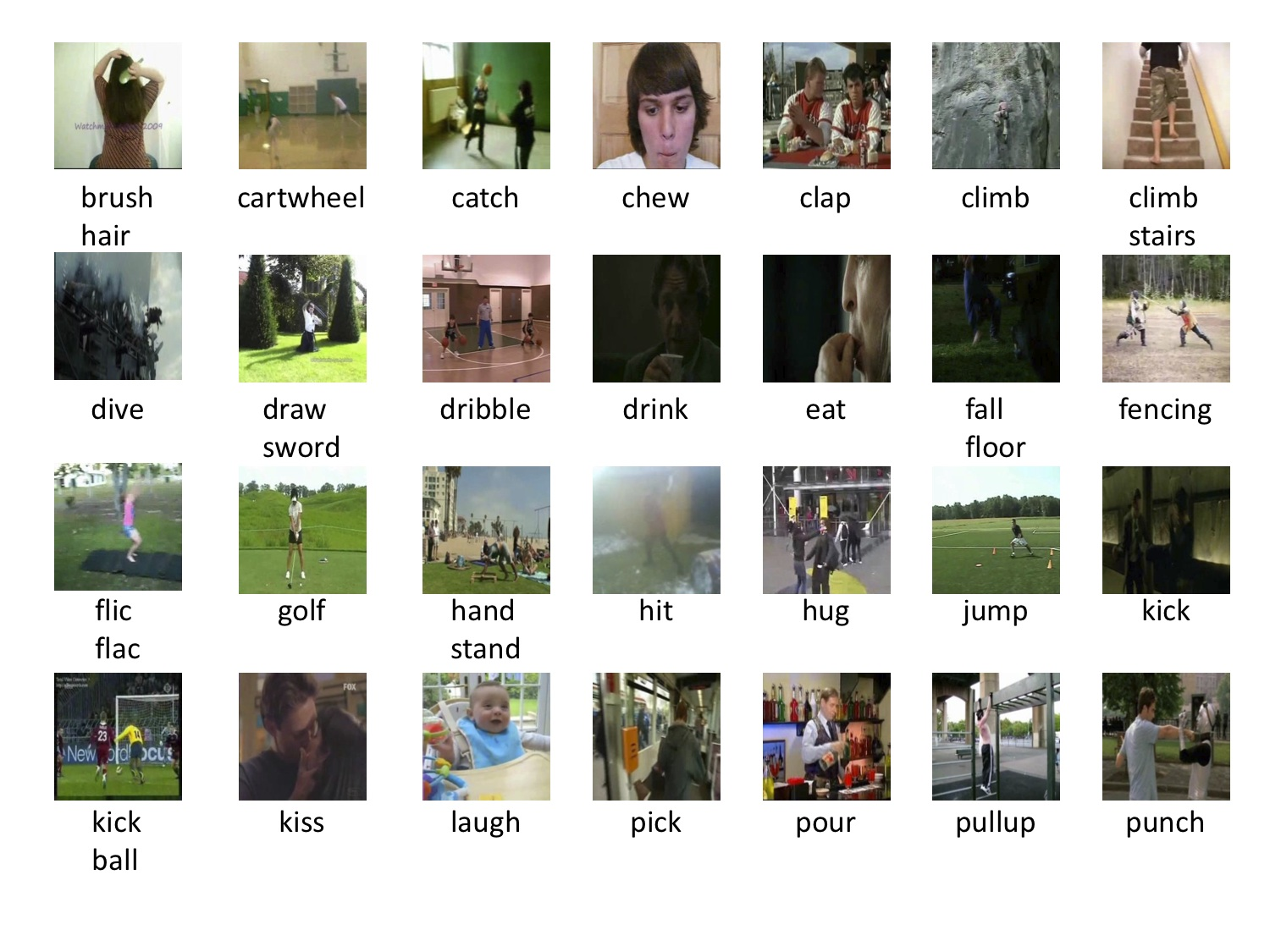}}\\
\subfloat[Hollywood2]{\includegraphics[width=.9\columnwidth]{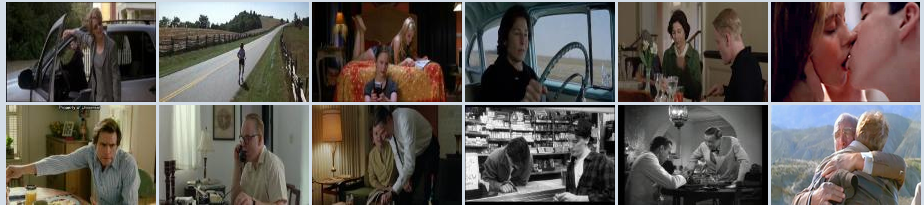}}\\
\subfloat[UCF101]{\includegraphics[width=.9\columnwidth]{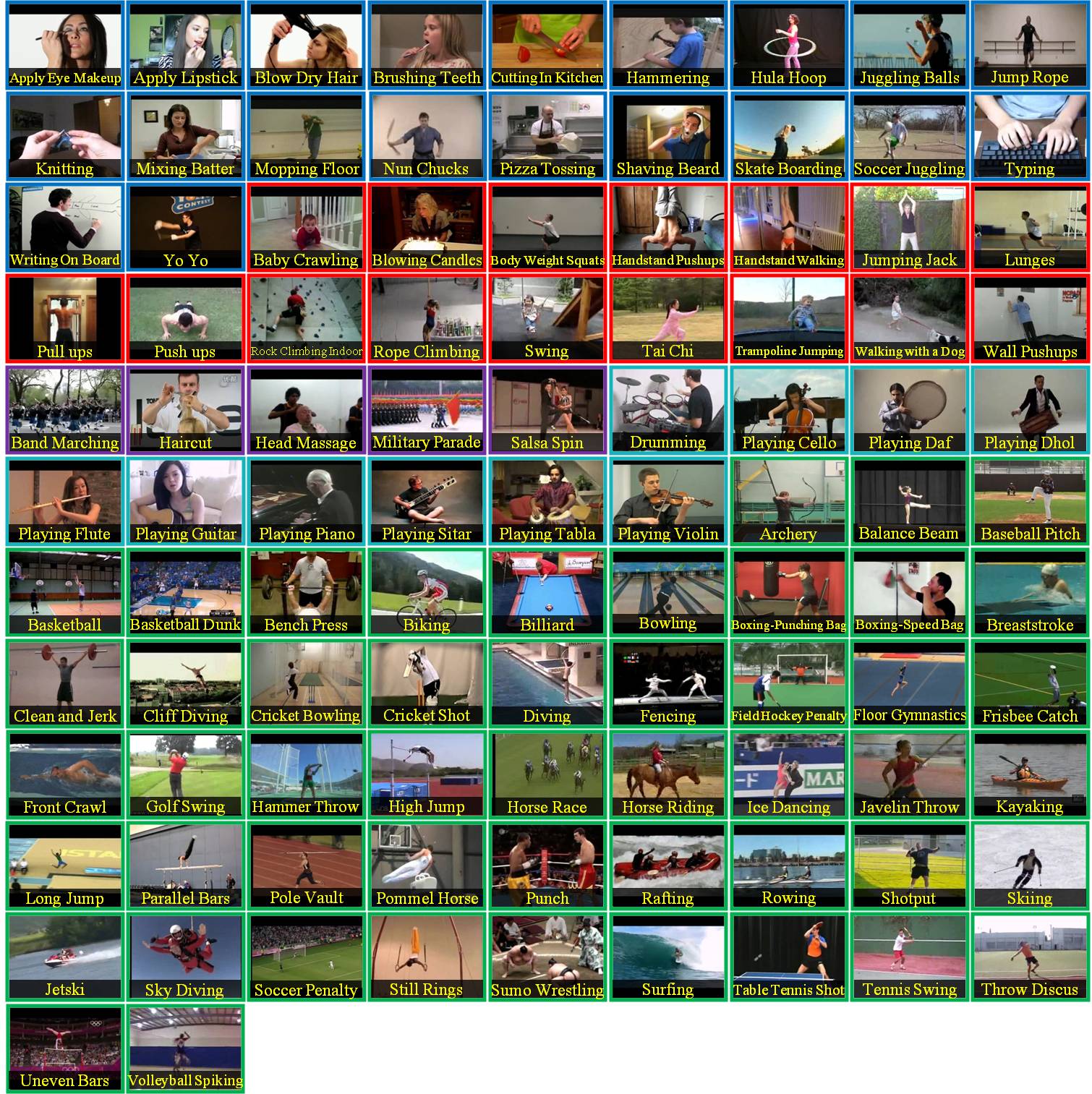}}\\
\subfloat[UCFSports]{\includegraphics[width=.9\columnwidth]{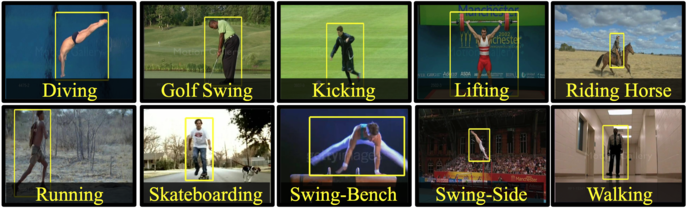}}
\caption{Example frames from (a) HMDB51 (b) Hollywood2  (c) UCF101 and (d) UCF-sports datasets from different action and activity classes.}
\label{fig.datasets}
\end{center}
\end{figure}

The rest of the experimental section is organised as follows. 
First in \secref{sec.exp.hrp.main} we provide a detailed evaluation of hierarchical rank pooling.
Then in \secref{sec.exp.disk.rp}, we evaluate the impact of discriminative rank pooling.
\secref{sec.exp.end.to.end} is dedicated to provide a detailed evaluation of end-to-end trainable rank pooled CNNs.
Finally, we compare with some state-of-the-art action recognition methods and position our contributions in~\secref{sec:experiment:soa}.
Implementation of our method is publicly available\footnote{\url{https://bitbucket.org/bfernando/hrp}}.

\subsection{Evaluating hierarchical rank pooling (HRP)}
\label{sec.exp.hrp.main}
First, we evaluate activity recognition performance using CNN features and hierarchical rank pooling (HRP) and then provide some detailed analysis.

\noindent
\textbf{Experimental details:} 
We utilize pre-trained CNNs without any fine-tuning. 
Specifically, for each video we extract activations from the VGG-16 \cite{Simonyan2014a} network's first fully connected layer (consisting of 4096 values, only from the central patch). 
We represent each video frame by this 4096 dimensional vector. Note that at this point, we do not use any ReLU \cite{Krizhevsky2012} non-linearity. As a result the frame representation vector contains both positive and negative components of the activations.

Unless otherwise specified, we use a window size $M_\ell$ of 20, with a stride $S_\ell$ of one and a hierarchy depth of two in all our experiments. 
We use a constant $C = 1$ parameter for SVR training (Lib-linear \cite{Fan2008}) to obtain the \rp-based temporal encoding as recommended in \cite{Fernando2016}. 
We test different non-linear SVM classifiers for the final classification always with $C = 1000$ (LibSVM \cite{Chang2011}) as this works well in practice. 
It should be noted that ideally, the best results can be obtained by cross-validation. 
However, as commonly done in state-of-the art action recognition methods \cite{wang2013action}, we use a fixed $C$ for LibSVM training.
In the case of multi-class classification, we use a one-against-rest approach and select the class with the highest score.
For rank pooling \cite{Fernando2015,Fernando2016} and trajectory extraction \cite{wang2013action} (in later experiments) we use the publicly available code from the authors.

\subsubsection{Comparing temporal pooling methods}
\label{sec.exp.hrp.eval}
\begin{table}[t]
\small
\centering 
\setlength{\tabcolsep}{3pt}
\begin{tabular}{lccc} \hline
{\sc Method} 		& Hollywood2 & HMDB51 & UCF101 \\ \hline
Average pooling 		& 40.9	& 37.1 & 69.3 \\ 
Max pooling 			& 42.4	& 39.1 & 72.5 \\ 
Tempo. pyramid (avg. pool)  	& 46.5  & 39.1 & 73.3 \\ 
Tempo. pyramid (max pool)  	& 48.7  & 39.8 & 74.8 \\ 
LSTM \cite{Srivastava2015}      & --    & 42.8 & 74.5 \\
Rank pooling 			& 44.2	& 40.9 & 72.2 \\ \hline
Recursive rank pooling          & 52.5	& 45.8 & 75.6 \\ 
Hierarchical rank pooling       & \textbf{56.8}	& \textbf{47.5} & \textbf{78.8} \\ \hline
Improvement & {\textcolor{red}{+8.1}} & {\textcolor{red}{+4.7}} & {\textcolor{red}{+4.0}} \\ \hline
\end{tabular}
\caption{Comparing several temporal pooling methods for activity recognition using VGG-16's fc6 features.}
\label{tab:baselines}
\end{table}

In this section we compare several temporal pooling methods using VGG-16 CNN features. We compare our hierarchical rank pooling with average-pooling, max-pooling, LSTM \cite{Srivastava2015}, two level temporal pyramids with mean pooling, two level temporal pyramids with max pooling, and vanilla rank pooling \cite{Fernando2015,Fernando2016}. To obtain a representation for average pooling, the average CNN feature activation over all frames of a video was computed. The max-pooled vector is obtain by applying the \emph{max} operation over each dimension of the CNN feature vectors from all frames of a given video. We also compare with a variant of hierarchical rank pooling called \emph{recursive rank pooling}, where the next layer's sequence element at time $t$ denoted by $\bx^{(\ell + 1)}_t$ is obtained by encoding \emph{all frames} of the previous layer sequence up to time $t$, \ie $\bx^{(\ell + 1)}_t = \phi( \langle\psi(\bx^{(\ell)}_1), \ldots, \psi(\bx^{(\ell)}_{t})\rangle)$ for $t = 2, \ldots, J_{\ell}$.

We compare these base temporal encoding methods on three datasets and report results in Table~\ref{tab:baselines}. 
Results show that the rank pooling method is only slightly better than max pooling or mean pooling when used with VGG16 features. 
We believe this is due to the limited capacity of rank pooling \cite{Fernando2015,Fernando2016}. 
Moreover, temporal pyramid seems to outperform rank pooling except for HMDB51 dataset.  
Moreover, as shown in Table~\ref{tab:baselines}, when we extend the rank pooling to \emph{recursive rank pooling}, we notice a jump in performance from 44.2\% to 52.5\% for Hollywood2 dataset and 40.9\% to 45.8\% for HMDB51 dataset. 
We also see a noticeable improvement in UCF101 dataset. 
Hierarchical rank pooling improves over rank pooling by a significant margin. 
The results suggest that it is important to exploit dynamic information in a hierarchical manner as it allows complicated sequence dynamics of videos to be expressed. 
To verify this, we also performed an experiment by varying the depth of the hierarchical rank pooling and reported results for one to three layers. 
Results are shown in Figure~\ref{fig:depth}.

\begin{figure}[t]
\centering 
 \subfloat[Hollywood2]{\includegraphics[width=0.5\linewidth]{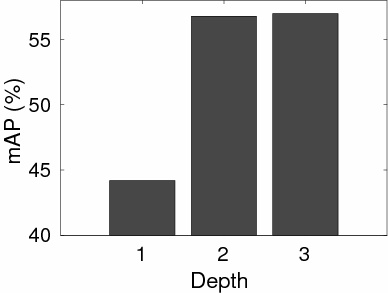}}
 \subfloat[HMDB-51]{\includegraphics[width=0.5\linewidth]{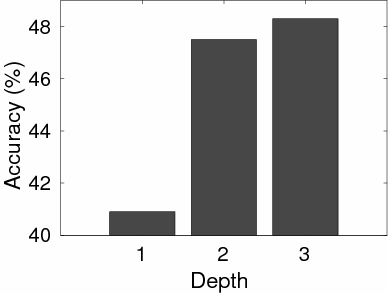}} 
 \caption{Activity recognition performance versus hierarchy depth on Hollywood2 and HMDB-51.}
 \label{fig:depth}
\end{figure}

As expected the improvement from depth of one to two is significant. Interestingly, as we increase the depth of the hierarchy to three, the improvement is marginal. Perhaps with only two levels, one can obtain a high capacity dynamic encoding.

\subsubsection{Evaluating the parameters of HRP}
Hierarchical rank pooling consists of two more hyper-parameters: (1) \textbf{window size} ($M_\ell$), \ie, the size of the video sub-sequences and (2) \textbf{stride} ($S_\ell$) of the video sampling. These two parameters control how many sub-sequences can be generated at each layer. In the next experiment we evaluate how performance varies with \textbf{window size} and \textbf{stride}. Results are reported in Figure~\ref{fig:stridewindow}(top). The window size does not seem to make a big impact on the results (1--2\%) for some datasets. However, we experimentally verified that a window size of 20 frames seems to be a reasonable compromise for all activity recognition tasks. The trend in Figure~\ref{fig:stridewindow}(bottom) for the stride is interesting. It shows that the best results are always obtained by using a small stride. Small strides generate more encoded sub-sequences capturing more statistical information.
\begin{figure}
\centering 
 {\includegraphics[width=0.3\linewidth,height=2cm]{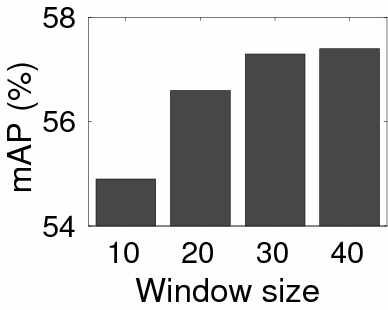}}
 {\includegraphics[width=0.3\linewidth,height=2cm]{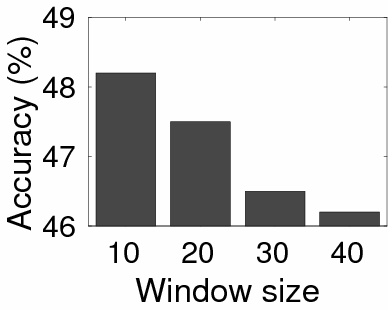}}
 {\includegraphics[width=0.3\linewidth,height=2cm]{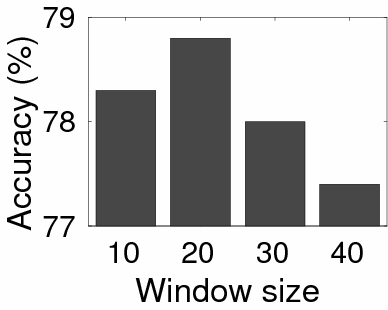}} \\ 
 
 \subfloat[Hollywood2]{\includegraphics[width=0.3\linewidth,height=2cm]{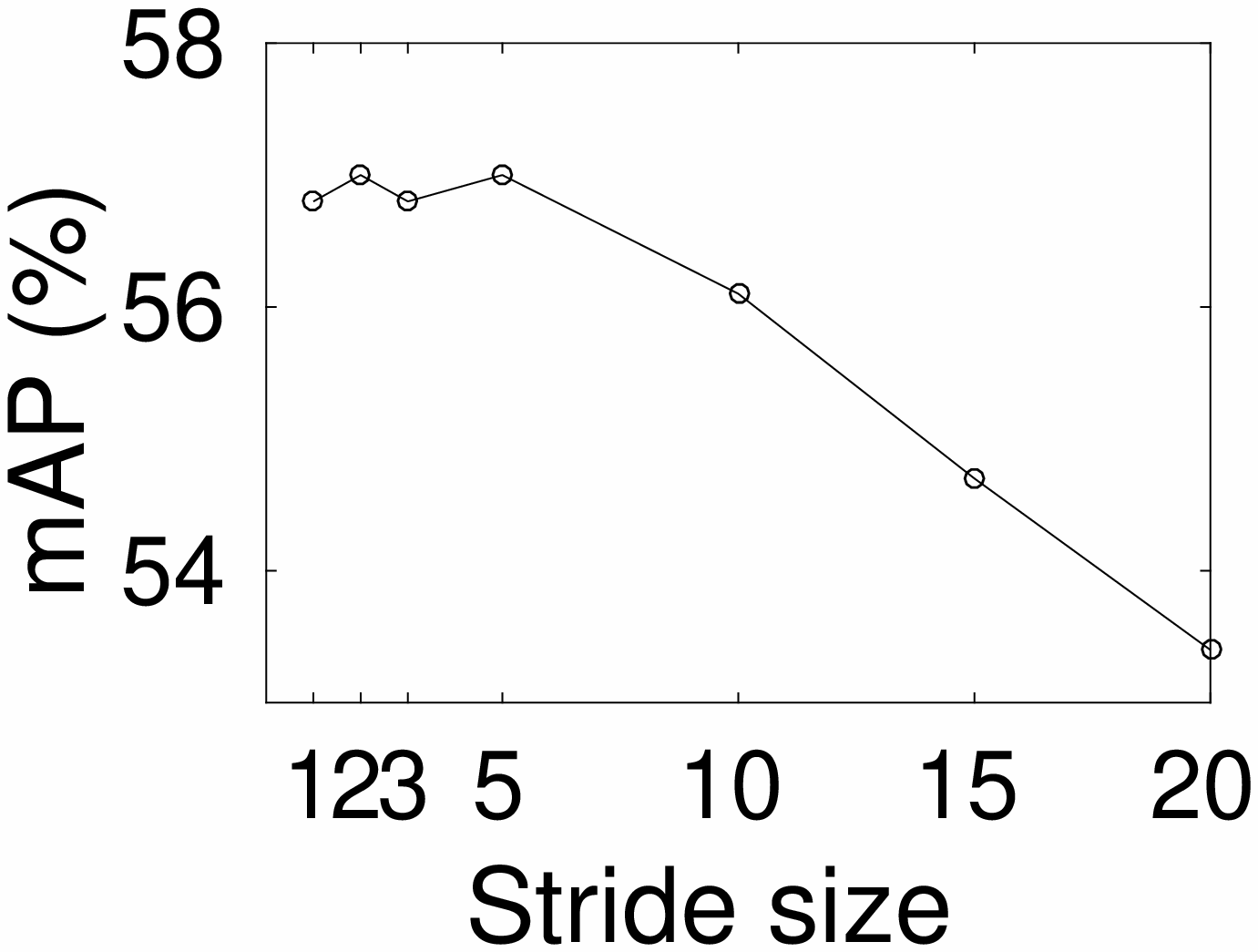}}
 \subfloat[HMDB-51]{\includegraphics[width=0.3\linewidth,height=2cm]{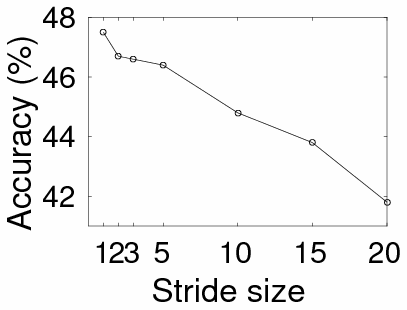}}
 \subfloat[UCF101]{\includegraphics[width=0.3\linewidth,height=2cm]{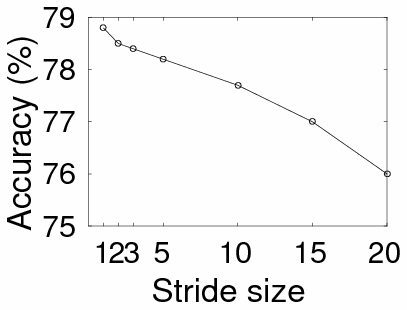}} 
 \caption{Activity recognition performance versus window size (top) and stride (bottom).}
 \label{fig:stridewindow}
\end{figure}
\subsubsection{The effect of non-linear feature maps on HRP}
Non-linear feature maps are important for modeling complex dynamics of an input video sequence. In this section we compare Sign Expansion Root (SER) feature map introduced in Section~\ref{sec:method:non-linear-maps} with the Signed Square Root (SSR) method, which is commonly used in the literature \cite{Perronnin2010}. Results are reported in Table~\ref{tab:featmaps.on.ranking}. As evident in the table, SER feature map is useful not only for hierarchical rank pooling, which gives an improvement of 6.3\% over SSR, but also for baseline rank pooling method, which gives an improvement of 6.8\%. This seems to suggest that there is valuable information in both positive and negative activations of fully connected layers. Furthermore, this experiment suggests that it is important to consider positive and activations separately for activity recognition.
\begin{table}
\small
\centering
\begin{tabular}{lcc} \hline
                                        &              & Hierarchical \\
{\sc Method} 				& Rank pooling & rank pooling \\ \hline
Signed square root (SSR) 		& 	44.2	& 50.5\\ 
Sign expansion root (SER)  		& 	51.0	& {\bf 56.8}\\  \hline
\end{tabular}
\caption{Effect non-linear feature maps during the training of rank pooling methods using Hollywood2 dataset.}
\label{tab:featmaps.on.ranking}
\end{table}
\subsubsection{The effect of non-linear kernel SVM on HRP}
In this experiment we evaluate several non-linear kernels that exist in literature and compare their effect when used with Hierarchical Rank Pooling method. We compare classification performance using different kernels (1) linear, (2) linear kernel with SSR,  (3) Chi-square kernel, (4) Kernelized SER  (5) combination of Chi-square kernel with SER. Results are reported in Table~\ref{tab:featmaps}. On all three datasets we see a common trend. First, the SSR kernel is more effective than not utilizing any kernel or feature map. Interestingly, on deep CNN features, Chi-square Kernel is more effective than SSR. Perhaps this is because the Chi-square kernel utilizes both negative and positive activations in a separate manner to some extent. The SER method seems to be the most effective kernel. Interestingly, applying SER feature map over Chi-square kernel seems to improve results further. We conclude that SER non-linear feature map is effective not only during the training of rank pooling techniques, but also for action classification specially when used with CNN activation features.   
\begin{table}
\small
\centering
\setlength{\tabcolsep}{3pt}
\begin{tabular}{l c c c } \hline
 				& Hollywood2 & HMDB51 & UCF101 \\
{\sc Kernel type} 		& (mAP \%) & (\%) & (\%) \\ \hline
Linear 				& 	45.1	& 	40.0 	 &  66.7\\ 
Signed square root (SSR) 	& 	48.6	& 	42.8 	 &  72.0\\ 
Chi-square kernel		& 	50.6	& 	44.2 	 &  73.8\\ 
Sign expansion root (SER) 	& 	54.0	& 	46.0	 &  76.6\\ 
Chi-square + SER		& \BF{56.8}	&  \BF{ 47.5}	 &  \BF{78.8}\\ \hline
\end{tabular}
\caption{Effect of non-linear SVM kernels on action classification with hierarchical rank pooling representation.}
\label{tab:featmaps}
\end{table}
Next we also evaluate the effect of non-linear kernels on final video representations when used with other pooling methods such as rank pooling, average pooling and max pooling. Results are reported in Table~\ref{tab:featmaps.on.other} on Hollywood2 dataset. A similar trend as in the previous table can be observed here. We conclude that our kernalized SER is useful not only for our hierarchical rank pooling method, but also for the other considered temporal pooling techniques. 
\begin{table}
\small
\centering 
\setlength{\tabcolsep}{2pt}
\begin{tabular}{l c c c c} \hline
{\sc Kernel type}			& Avg. pool & Max pool & Rank pool & Ours \\ \hline
Linear 		 			& 	38.1	& 	39.6 	 &  33.3 & 45.1\\ 
Signed square root (SSR) 		& 	38.6	& 	38.4 	 &  35.3 & 48.6\\ 
Chi-square kernel			& 	39.9	& 	41.1 	 &  40.8 & 50.6\\ 
Sign expansion root (SER)  		& 	39.4	& 	41.0	 &  37.4 & 54.0\\ 
Chi-square + SER			& \BF{40.9}	&   \BF{42.4}	 &  \BF{44.2} & \BF{56.8}\\ \hline
\end{tabular}
\caption{Effect of non-linear kernels on other pooling methods using Hollywood2 dataset (mAP \%).}
\label{tab:featmaps.on.other}
\end{table}
\subsubsection{Combining hierarchical rank pooled CNN features with improved trajectory features}
\label{sec:experiment:trj}
In this experiment we combine hierarchical rank pooled CNN features with the Improved Dense Trajectory (IDT) features (MBH, HOG, HOF) \cite{wang2013action}. 
The objective of this experiment is to show the complimentary nature of IDT and hierarchical rank pooled CNN features.
IDT are encoded with Fisher vectors \cite{Perronnin2010} at the frame level and then temporally encoded with rank pooling. 
Due to the very high dimensional nature of Fisher vectors, it is not practical to use hierarchical rank pooling over Fisher vectors. 
We utilize a Gaussian mixture model of 256 components to create the Fisher vectors. To keep the dimensionality manageable, we halve the size of each descriptor using PCA. This is exactly the same setup used by Fernando~\etal \cite{Fernando2015,Fernando2016}. For each dataset we report results on HOG, HOF and MBH features obtained with the publicly available code of rank pooling \cite{Fernando2015,Fernando2016}.
We construct a kernel gram matrix for each feature type (HOG, HOF, MBH, and CNN) and take the averaging of the kernels to fuse features. 
Results are shown in Table~\ref{tab:trjcnn}.
Hierarchical rank pooled (CNN) outperforms trajectory based HOG features on all three datasets. Furthermore, on UCF101 dataset, Hierarchical rank pooled (CNN) outperforms rank pooled HOF features. Nevertheless, trajectory based MBH features still dominate the best results for an individual feature. The combination of rank pooled trajectory features (HOG + HOF + MBH) with hierarchically rank pooled CNN features gives a significant improvement.  It is interesting to see that the biggest improvement is obtained in Hollywood2 dataset. On UCF-101 dataset the combination brings us an improvement of 4.2\% over rank pooled trajectory features. We conclude that our hierarchical rank pool features are complimentary to trajectory-based rank pooling. 
\begin{table}[t]
\small
\centering
\setlength{\tabcolsep}{2pt} 
\begin{tabular}{l c c c } \hline
{\sc Method} 				& Hollywood2 & HMDB51  & UCF101 \\ \hline
RP. (HOG) 			& 53.4		& 	44.1 	 &  	72.8	\\ 
RP. (HOF) 			& 64.0		& 	53.7 	 &  	78.3	\\ 
RP. (MBH) 			& \BL{65.8} 	& 	\BL{53.9}&  	\BL{82.6}\\ 
RP. (ALL) 			& 68.5		& 	60.0	 &  	86.5	\\ 
RP. (CNN) 			& 44.2		& 	40.9	 &  	72.2	\\ 
RP. (ALL+CNN	)		& 71.4		& 	63.0	 &  	88.1	\\  \hline
HRP. (CNN) 			& 56.8		&       47.5	 &  	78.8	\\ 
RP. (ALL)+ HRP (CNN) 	& \textbf{74.1}	&\textbf{65.0}	 &\textbf{90.7}	\\  \hline
\end{tabular}
\caption{Combining CNN-based Hierarchical Rank Pooling (HRP) with improved trajectory features encoded with Fisher vectors and Rank Pooling(RP).}
\label{tab:trjcnn}
\end{table}
\subsubsection{Combining with trajectory features}
We also apply hierarchical rank pooling over improved dense trajectories which are encoded with the bag-of-words.
For this experiment, we use MBH features and use a dictionary of size 4096 which is constructed with K-means.
Results are reported in Table~\ref{tbl.idt.bow}.
\begin{table}[h!]
\centering
\begin{tabular}{lcc}\hline
Method & UCF101 Acc. (\%) & HMDB51 Acc. (\%) \\
\hline
Average pooling  		& 72.3 & 45.0\\
Max pooling 			& 71.5 & 43.1\\
Rank pooling			& 77.5 & 48.1 \\
Hierarchical rank pooling 	& 82.1 & 54.2\\
\hline
\end{tabular}
\caption{Action classification performance using IDT (MBH) features encode with BOW (4096 dictionary) using different temporal pooling methods and SVM classifiers.}
\label{tbl.idt.bow}
\end{table}
As before, both average pooling and max pooling perform worst than the rank pooling method.
Hierarchical rank pooling obtains large improvement over rank pooling.
On HMDB51 dataset, the improvement over rank pooling is about 6\%. HRP obtains an improvement of 4.6\% on UCF101 over rank pooling.
It is interesting to see the impact of hierarchical rank pooling over deep features as well as traditional hand-crafted features such as dense trajectory features and bag-of-words encoding.
We conclude that the hierarchical rank pooling is effective not only on recent deep features, but also with more traditional IDT-based bag-of-words features.

\subsubsection{The impact of residual network features on HRP}
In this experiment, we evaluate the impact of Residual Network Features \cite{He2016} on action recognition using UCF101 and HMDB51 datasets.
Results for max pooling, average pooling, Rank pooling, and Hierarchical Rank pooling with ResNet features are shown in Table~\ref{tab.ucf101.resnet} for UCF101 and HMDB51 datasets.
For this analysis, we extract frame level ResNet features from the output of final pooling layer which has a dimensionality of 2048.
We compare our hierarchical rank pooling with max pooling method. For rank pooling we obtain classification accuracy of 84.0\% only using frame-level ResNet features on UCF101.
This is an improvement of 5.3\% over VGG-16 features. Similarly, for max pooling we obtain 78.8 \% which is an improvement of 6.3 \% over VGG-16.
Similar trends can be observed for HMDB51 dataset. 
In fact, for HMDB51, it seems the improvement from VGG-16 to ResNet features is significant (11.2 \% for average pooling, 11.1 \% for max pooling,  13.8 \% for rank pooling and 9.8 \% for hierarchical rank pooling).

In another experiment, we also used publicly available ResNet-152 networks that are finetuned for RGB stream~\cite{Feichtenhofer2016}. 
Then only using the center crop of UCF101 frames, we extract 2048 dimensional features per frame and experiment with several baseline methods.
For RNN and LSTM baselines, we use Keras~\cite{chollet2015keras} with hidden size of 256.
We report results in Table~\ref{tab.ucf101.resnet.ox}.
Interestingly, simple RNN and LSTM methods does not outperform max pooling or the average pooling results.
Rank pooling is better than max pooling while hierarchical rank pooling is significantly better than rank pooling.

We conclude that ResNet feature \cite{He2016} are useful for action and activity recognition and our proposed hierarchical rank pooling method is complimentary to both VGG-16 features \cite{Simonyan2014a} as well as ResNet features. 

\begin{table}[t]
\centering
\begin{tabular}{lcc}\hline
Method & UCF101 Acc. (\%) & HMDB51 Acc. (\%) \\
\hline
Average pooling  		& 76.5 & 48.3\\
Max pooling 			& 78.8 & 50.2\\
Rank pooling			& 81.0 & 54.7 \\
Hierarchical rank pooling 	& \textbf{84.0} & \textbf{57.3}\\
\hline
\end{tabular}
\caption{Action classification performance using non-fine-tuned ResNet features \cite{He2016} using different temporal pooling methods and SVM classifiers.}
\label{tab.ucf101.resnet}
\end{table}

\begin{table}[t]
\centering
\begin{tabular}{lcc}\hline
Method & UCF101  \\
\hline
Simple RNN 			& 74.8  \\
Simple LSTM 			& 75.9  \\
Stacking of two LSTMs 		& 75.3  \\
Average pooling  		& 79.1    \\
Max pooling 			& 81.3   \\
Rank pooling			& 82.1 \\
Hierarchical rank pooling 	& \textbf{85.6}\\
\hline
\end{tabular}
\caption{Action classification performance using fine-tuned ResNet-152 features \cite{Feichtenhofer2016} using spatial stream. We use different temporal pooling methods and compare results on UCF101 dataset.}
\label{tab.ucf101.resnet.ox}
\end{table}

\subsubsection{Confusions with the use of residual network features and HRP}
We also analyse the confusions made by ResNets when pooled using max operator and hierarchical rank pooling (see Figure~\ref{fig.confuse}). 
\begin{figure}[h!]
\centering 
 \includegraphics[width=0.49\linewidth]{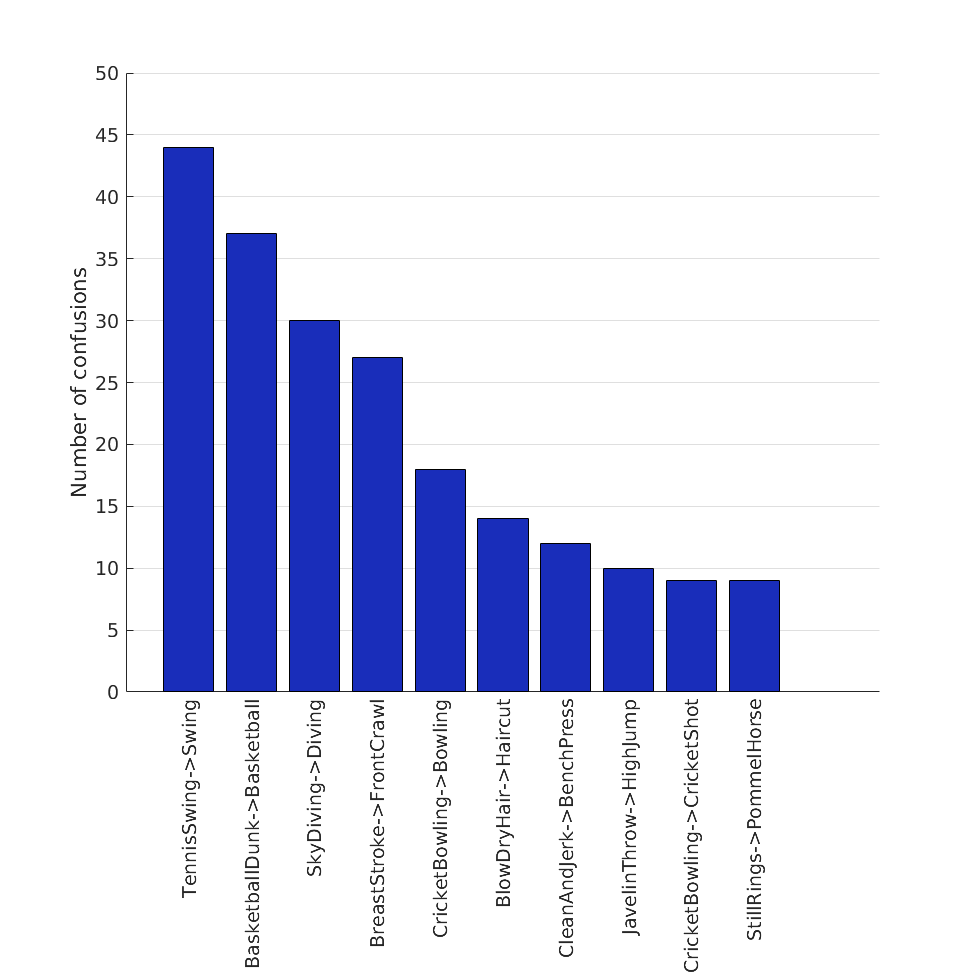}
 \includegraphics[width=0.49\linewidth]{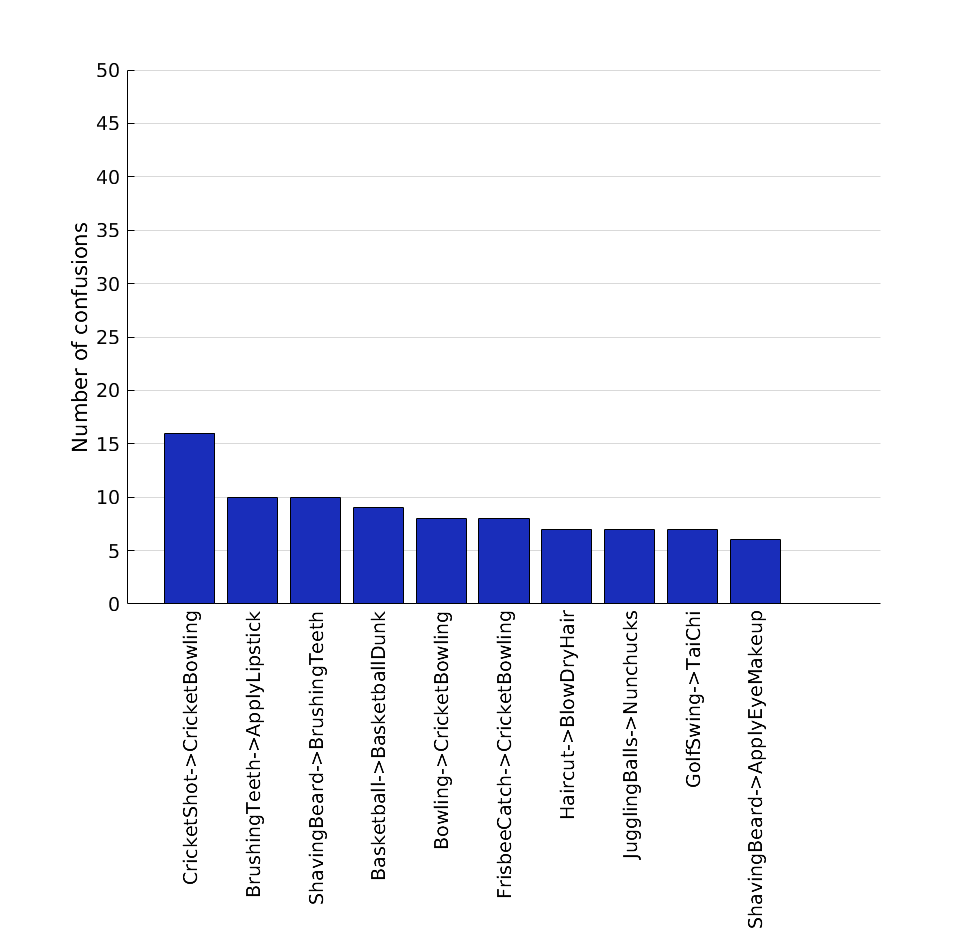}
 \caption{The most confused classes for max pooling (left) and hierarchical rank  pooling in the right. Max pooling makes most confusions.}
 \label{fig.confuse}
\end{figure}
The most confusing category for max pooling is \textit{Swing} for \textit{Tennis swing} (44 times) and \textit{Basketball} for \textit{Basketball-Dunk} (37 times) (--see Figure~\ref{fig.confuse.hrp} left). 
The most confusing  for hierarchical rank pooling is \textit{Cricket-Bowling} for \textit{Cricket-Shot} which happens only 16 times (--see Figure~\ref{fig.confuse.hrp} right). Generally, from the dynamics point of view, it is very hard to distinguish Cricket bowling from Cricket-Shot as indeed the Cricket-Shot just follows after Cricket-bowling. In particular, in many cases Cricket-bowling can be observed for Cricket-Shot video clips in UCF101 dataset. 
\begin{figure*}
\centering 
\includegraphics[width=0.45\linewidth]{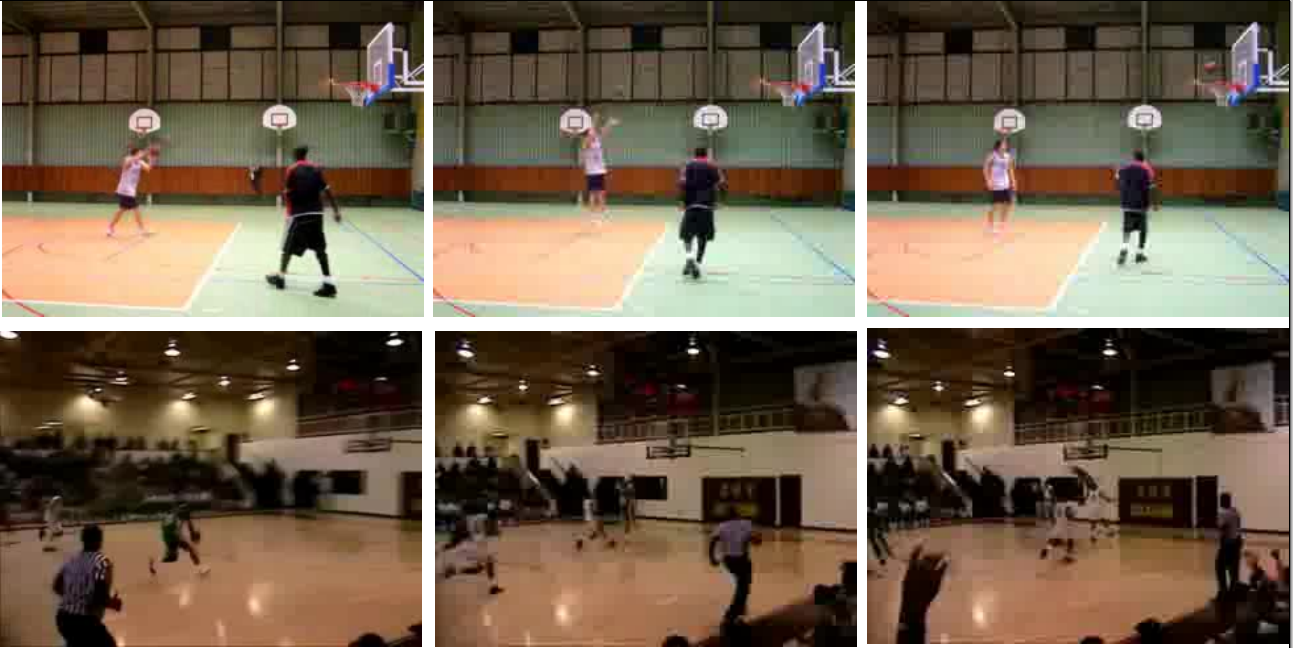}
\includegraphics[width=0.45\linewidth]{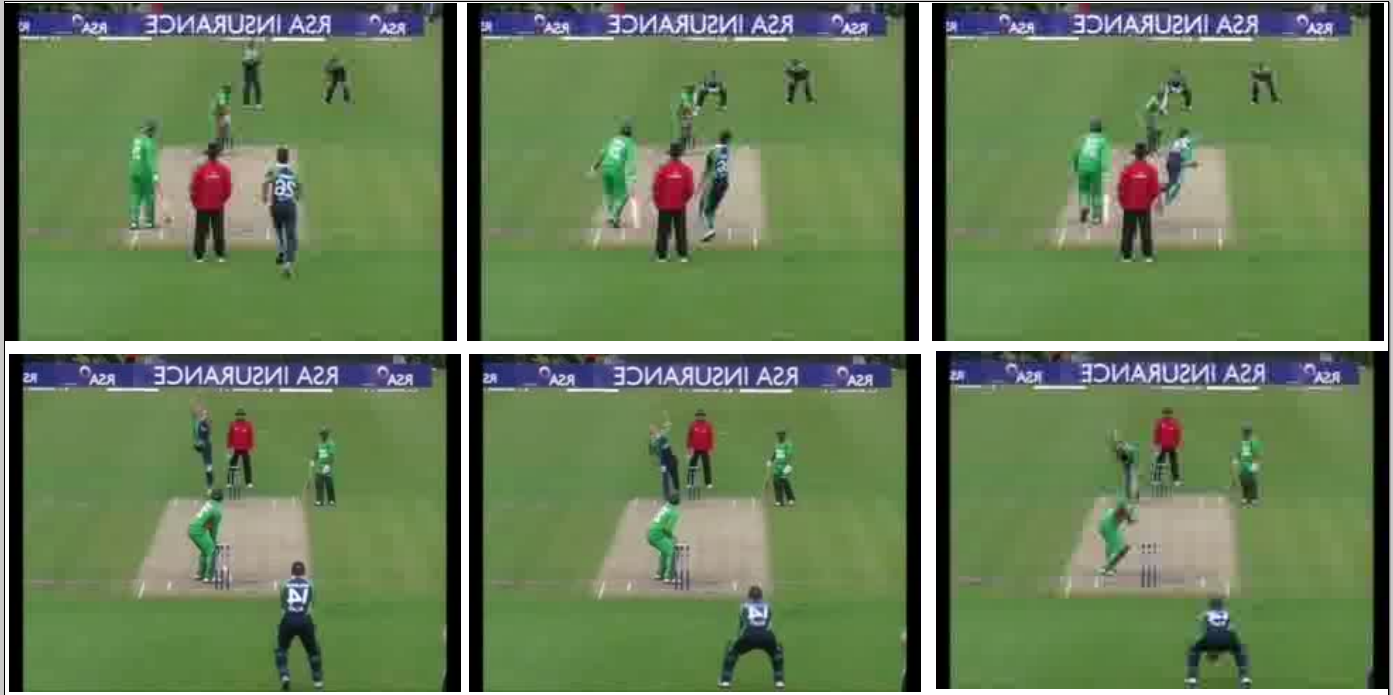} 
 \caption{Some of the most confusing classes for max pooling (left) and hierarchical rank pooling (right) respectively using ResNet features.}
 \label{fig.confuse.hrp}
\end{figure*}
\subsubsection{Impact of mid-level pooled features}
In this experiment, we evaluate the impact of low-level, learned mid-level, and higher level features of the hierarchical rank pooling.
We use non-fine-tuned ResNet-150 features~\cite{He2016} as the frame representation.
As before we use a window size of 20 and stride of 1.
After applying three layered hierarchical rank pooling, we use the first/second layer mid-level features as the mid-level sequence representations.
We randomly pick a mid-level feature vector to represent the entire video sequence.
To compare, we also pick a single frame feature to represent a video.
Furthermore, we randomly select 39 frames from each video and apply temporal max pooling and temporal average pooling as baselines.
We evaluate the impact of each mid-level feature and position the results with respect to the highest level hierarchical rank-pooled feature.
We repeat each experiment 10 times and report the mean and standard deviation in Table~\ref{impact.mid.level}.
\begin{table}[h!]
\centering
\begin{tabular}{lc}\hline
Level &  HMDB51 Acc. (\%) \\
\hline
0 - frame level			 & 38.8 $\pm$ 1.1\\ 	
temporal max pooling (39 frames) & 49.2 $\pm$ 1.6\\
temporal avg. pooling (39 frames)& 46.9 $\pm$ 0.5\\
$1^{st}$ layer mid-level feature & 41.9 $\pm$ 1.4\\
$2^{nd}$ layer mid-level feature & 47.5 $\pm$ 0.6\\
$3^{rd}$ layer Hierarchical rank pooling& 57.4\\
\hline
\end{tabular}
\caption{The impact of low-level, mid-level, and highest level sequence representation using hierarchical rank pooling method.}
\label{impact.mid.level}
\end{table}
Clearly, frame level feature performs the worst. This is expected.
Interestingly, using just a single random frame, we are able to obtain a mean classification accuracy of 38.8 \%.
First layer mid-level feature is better than frame level representation which obtains 41.9 \%. 
The second layer mid-level feature is even better which obtains 47.5 \%. 
This is an indication of the impact of mid-level dynamics.
Note that the temporal resolution of the first layer feature is 20 frames while the second layer mid-level feature has a resolution of 39 frames.
Most interestingly, the highest level features obtain 57.4 \% which is significant. 
However, the highest level feature has the full temporal resolution.
These results suggest that indeed, the hierarchical rank pooling is capable of capturing low-level, mid-level and higher level dynamics. 
The highest level temporal dynamics captured by HRP improves the activity recognition performance significantly.
\subsection{The effect of discriminative rank pooling}
\label{sec.exp.disk.rp}
In this experiment, we use discriminative rank pooling in the final layer of the hierarchical rank pooling network. In this case we first construct the sequence for the final layer $X^{(L)}$ and apply SSR feature map. Then we feed forward this sequence through the parameterized non-linear transform $\psi(W \bx^{(L)}_t)$, temporal encoder $\phi(\widetilde{X}^{(L)})$, and apply the classifier to get a classification score. During training we propagate errors back to the parametric non-linear transformation layer $\psi(\cdot)$ and perform a parameter update. We implement this optimization in a GPU.

We use MatConvNet \cite{Vedaldi2015} with stochastic gradient descent with variable learning rate starting at $10^{-3}$ and decreased to $10^{-5}$ in a logarithmic manner over epochs. We also use a momentum term of 0.9 and a weight decay of 0.0005. Our layer is implemented in matlab with GPU support. We evaluate the effect of this method only on the largest datasets, the HMDB51 and UCF101. We first construct the first layer sequence using hierarchical rank pooling. Then we learn the parameters $W$ using the labelled video data while keeping the CNN parameters fixed. We initialize the $W$ matrix to the identity and the classifier parameters to those obtained from the linear SVM classifier. Results are reported in Table~\ref{tab.discriminative}. We improve results by 2.4\% and 2.6\% over hierarchical rank pooling and a significant improvement of 9.0\% and 9.2\% over rank pooling using HMDB51 and UCF101 datasets respectively. During test time, we process a video at 120 frames per second.
\begin{table}
\small
\centering 
\begin{tabular}{l c c} \hline
{\sc Method} 				& HMDB51 	& UCF101 \\ \hline
Rank pooling 		 		& 40.9		& 	72.2\\ 
Hierarchical rank pooling 		& 47.5 		& 	78.8\\
Discriminative hierarchical rank pooling& \textbf{49.9 $\pm$ 0.08}  & 	\textbf{81.4 $\pm$ 0.04} \\ \hline
\end{tabular}
\caption{Effect of learning discriminative dynamics for hierarchical rank pooling on the HMDB51 and UCF101 dataset.}
\label{tab.discriminative}
\end{table}
\subsection{Comparing the effect of end-to-end trainable rank pooled CNN.}
\label{sec.exp.end.to.end}
\begin{table}[t]
\centering
\begin{tabular}{lc}\hline
Method & Acc. (\%)\\
\hline
Average pooling + svm & 67.1 \\
Max pooling + svm & 66.0 \\
Rank pooling + svm & 66.4 \\
Average pooled-cnn-end-to-end & 70.4 \\
Max pooled-cnn-end-to-end & 71.2 \\
Frame-level fine-tuning & 69.8 \\
Frame-level fine-tuning + Rank pooling & 72.9 \\
Rank-pooled-cnn-end-to-end  & \textbf{87.1} \\
\hline
\end{tabular}
\caption{Classification accuracies for action recognition on the ten-class UCF-sports dataset \cite{Rodriguez2008} using end-to-end video representation learning with rank pooling.}
\label{tab.ucfsports}
\end{table}

\begin{table*}[t!]
\centering
\begin{tabular}{lcccccccc}\hline
class & avg+svm & max+svm & rankpool+svm & avg+cnn & max+cnn & fn & fn+rankpool &  rankpool+cnn \\ \hline
AnswerPhone & 23.6 & 19.5 & \BF{35.3} & 29.9 & 28.0 & 27.4 & 34.3 & 25.0 \\
DriveCar & \BF{60.9} & 50.8 & 40.6 & 55.6 & 48.6 & 48.1 & 50.4 & 56.9 \\
Eat & 19.7 & 22.0 & 16.7 & \BF{27.8} & 22.0 & 21.1 & 23.1 & 24.2 \\ 
FightPerson & \BF{45.6} & 28.3 & 28.1 & 26.6 & 17.6 & 18.4 & 20.4 & 30.4 \\ 
GetOutCar & 39.5 & 29.2 & 28.1 & 48.9 & 43.8 & 43.1 & 45.3 & \BF{55.5} \\
HandShake & 28.3 & 24.4 & 34.2 & 38.4 & \BF{40.0} & 39.4 & 39.5 & 32.0 \\
HugPerson & 30.2 & 23.9 & 22.1 & 25.9 & 26.6 & 26.1 & 30.3 & \BF{33.2} \\
Kiss & 38.2 & 27.5 & 36.8 & 50.6 & 45.7 & 44.9 & 45.6 & \BF{54.2} \\
Run & 55.2 & 53.0 & 39.4 & 59.6 & 52.5 & 52.4 & 52.9 & \BF{61.0} \\
SitDown & 30.0 & 28.8 & 32.1 & 30.6 & 30.0 & 29.7 & 34.4 & \BF{39.6} \\
SitUp & 23.0 & 20.2 & 18.7 & 23.8 & \BF{26.4} & 24.1 & 25.1 & 25.4 \\
StandUp & 34.6 & 32.4 & 39.9 & 37.4 & 34.8 & 34.4 & 34.8 & \BF{49.9} \\ \hline
mAP & 35.7 & 30.0 & 31.0 & 37.9 & 34.7 & 34.1 & 36.3 & \BF{40.6} \\\hline
\end{tabular}
\caption{Classification performance in average precision for activity recognition on the Hollywood2 dataset \cite{Laptev2008} using end-to-end video representation learning with rank pooling.}
\label{tab.hollywood2}
\end{table*}

In this section we evaluate the effectiveness of end-to-end video representation learning with rank-pooling introduced in section~\ref{sec.endtoend}.
Due to the computational complexity, we only use moderate (Hollywood2) and small scale (UCF sports) action recognition dataset for evaluation.
We compare our end-to-end training of the rank-pooling network against the following baseline methods.

\textbf{avg pooling + svm:} We extract FC7 feature activations from the pre-trained Caffe reference model \cite{Jia2014} using MatConvNet \cite{Vedaldi2015} for each frame of the video. Then we apply temporal average pooling to obtain a fixed-length feature vector per video (4096 dimensional). Afterwards, we use a linear SVM classifier (LibSVM) to train and test action and activity categories.

\textbf{max pooling + svm:} Similar to the above baseline, we extract FC7 feature activations for each frame of the video and then apply temporal max pooling to obtain a fixed-length feature vector per video. Again we use a linear SVM classifier to predict action and activity categories.

\textbf{rank pooling + svm:} We extract FC7 feature activations for each frame of the video. We then apply time varying mean vectors to smooth the signal as recommended by \cite{Fernando2015}, and L2-normalize all frame features. Next, we apply the rank-pooling operator to obtain a video representation using publicly available code \cite{Fernando2015}. We use a linear SVM classifier applied on the L2-normalized representation to classify each video.

\textbf{frame-level fine-tuning (fn):} We fine-tune the Caffe reference model on the frame data considering each frame as an instance from the respective action category. Then we sum the classifier scores from each frame belonging to a video to obtain the final prediction.

\textbf{frame-level fine-tuning + rank-pooling (fn+rankpool):} We use the pre-trained model as before and fine-tune the Caffe reference model on the frame data considering each frame as an instance from the respective action category. Afterwards, we extract FC7 features from each video (frames). Then we encode temporal information of fine-tuned FC7 video data using rank-pooling. Afterwards, we use soft-max classifier to classify videos.

\textbf{end-to-end baselines:} We also compare our method with end-to-end trained max and average pooling variants. Here the pre-trained CNN parameters were fine-tuned using the classification loss.

The first five baselines can all be viewed as variants of the CNN-base temporal pooling architecture of~\figref{fig:cnnnet}. The differences being the pooling operation and whether end-to-end training is applied.

We compare the baseline methods against our rank-pooled CNN-based temporal architecture where training is done end-to-end. We do not sub-sample videos to generate fixed-length clips as typically done in the literature (e.g., \cite{Simonyan2014,Tran2015}). Instead, we consider the entire video during training as well as testing. We use stochastic gradient descent method without batch updates (i.e., each batch consists of a single video). We initialize the network with the Caffe reference model and use a variable learning rate starting from 0.01 down to 0.0001 over 60 epochs. We also use a weight decay of 0.0005 on an L2-regularizer over the model parameters. We explore two variants of the learning algorithm. In the first variant we use the diagonal approximation to the rank-pool gradient during the back-propagation. In the second variant we use the full gradient update, which requires computing the inverse of matrices per video (see \secref{sec.opt.diff}). For the UCF-sports dataset we use the cross-entropy loss for all CNN-based methods (including the baselines). Whereas for the Hollywood2 dataset, where performance is measured by mAP (as is common practice for this dataset), we use the hinge-loss.

Results for experiments on the UCF-sports dataset are reported in~\tabref{tab.ucfsports}. Let us make several observations. First, the performance of max, average and rank-pooling are similar when CNN activation features are used without end-to-end learning. Perhaps increasing the capacity of the model to better capture video dynamics (say, using a non-linear SVM) may improve results perhaps a future work. Second, end-to-end training helps all three pooling methods. However, the improvement obtained by end-to-end training of rank-pooling is about {\bf 21\%}, significantly higher than the other two pooling approaches. Moreover, the performance using the diagonal approximation is 87.0\% which is very close to the full gradient based approach. This suggests that the diagonal approximation is driving the parameters in a desirable direction and may be sufficient for a stochastic gradient-based method. Last, and perhaps most interesting, is that using state-of-the-art improved trajectory \cite{wang2013action} features (MBH, HOG, HOG) and Fisher vectors \cite{Perronnin2010} with rank-pooling \cite{Fernando2015} obtains 87.2\% on this dataset. This result is comparable with the results obtained with our method using end-to-end feature learning. Note, however, that the dimensionality of the feature vectors for the state-of-the-art method are extremely high (over 50,000 dimensional) compared to our 4,096 dimensional feature representation.

We now evaluate activity recognition performance on the Hollywood2 dataset. Results are reported in \tabref{tab.hollywood2} as average precision performance for each class and we take the mean average precision (mAP) to compare methods. As before, for this task, the best results are obtained by end-to-end training using rank-pooling for temporal encoding. The improvement over non-end-to-end rank pooling is {\bf 9.6} mAP. 
One may ask whether this performance could be achieved without end-to-end training but just fine-tuning the frame-level features. Simple frame-level fine-tuning obtains only 34.1 mAP (see Table~\ref{tab.hollywood2} with the column denoted by fn) while frame-level fine-tuning + rank-pooling obtains 36.3 mAP (see Table~\ref{tab.hollywood2} with the column denoted by fn+rankpool). Our end-to-end method obtains better results (40.6 mAP) compared to frame-level fine-tuning and fine-tuning with rank-pooling.

\subsection{Comparing to the state-of-the-art}
\label{sec:experiment:soa}
In this section we position our paper with respect to the current state-of-the-art performance in action recognition using standard datasets. 
We perform a series of experiments using hierarchical rank pooled deep cnn features for UCF101 and HMDB51 datasets.
We use two types of cnn features, one extracted from VGG-16-CNN architecture and the other extracted from ResNet architecture.
We also experimented with discriminative hierarchical rank pooling.
To further improve results, we use rank pooled \cite{Fernando2016} improved dense trajectory features (IDT) \cite{wang2013action} and optical-flow-based \cite{Brox2004} deep features for UCF101 and HMDB51 datasets.   
It should be emphasized, we choose parameters for hierarchical rank pooling based on the prior experimental results reported in Figures~\ref{fig:depth} and~\ref{fig:stridewindow} for each dataset, \ie, \emph{without} use of any grid search. 
As in \cite{Fernando2015,hoi2014} we use data augmentation only for Hollywood2 and HMDB51.
Results are reported in the following Table~\ref{tab:soa}.

When ResNet (RGB) features are combined with IDT, our hrp-based method obtains staggering 93.1\% on UCF101. 
Furthermore, if we add optical flow-based features  (similar to RGB-based hierarchical rank pooling), we obtain 93.6\% classification performance on UCF101 dataset. 
Only using ResNet-based RGB data and Optical Flow data, hierarchical rank pooling with default settings obtains 90.6\% on UCF101. 
Similarly, on HMDB51 dataset, hierarchical rank pooled ResNet (RGB + Optical-flow) features obtains 63.1\%. 
When we combine that with IDT features, for HMDB51 dataset we obtain 69.4 \% which is par with the state-of-the art for this dataset. 
On Hollywood2 dataset, hierarchical rank pooled VGG-16 features are combined with IDT to obtain state-of-the art 76.7 mAP. 
This is a significant improvement over rank pooling \cite{Fernando2015} method.

Because, different methods used different information such as optical flow features, different motion representations, different object models and trajectory-based features, it is difficult to compare methods in a purely fair manner using the published results alone. However, from these results obtained in Table~\ref{tab:soa}, we conclude that our sequence encoding method and end-to-end learning method are complimentary to existing techniques and video data and features.
\begin{table*}[t]
\centering
\small
\setlength{\tabcolsep}{3pt}
\begin{tabular}{l l c c c} \hline
Method & Feature & Holly.2 & HMDB51 &   UCF101 \\ \hline
\textit{hrp} & ResNet (RGB+Opt.Flow)  + \IDT		& --   & \textbf{69.4}   & \textbf{93.6}  \\ 
\textit{hrp} & ResNet (RGB) + \IDT			& --   & 68.9   & 93.1  \\ 
\textit{dhrp}& VGG-16 (RGB) + \IDT			& --   & 68.1 & 91.4  \\ 
\textit{hrp} & VGG-16 (RGB) + \IDT			& \textbf{76.7}   & 66.9   & 91.2  \\
\textit{hrp} & ResNet(RGB+Opt.Flow)	 		& --   & 63.1  & 90.6  \\  \hline
\textit{Zha~\etal \cite{Zha2015}} 	& VGG-19 (RGB)+\IDT		&    &    & 89.6  \\
\textit{Ng~\etal \cite{Yue-HeiNg2015}} & GoogLeNet (RGB + Opt.FLow)		&    &    & 88.6 \\
\textit{Simonyan~\etal \cite{Simonyan2014}} & CNN-M-2048  (RGB + Opt.FLow)		&    & 59.4 & 88.0 \\ 
\textit{Wang~\etal \cite{Wang2015}} & CNN-M-2048  (RGB + Opt.FLow) + \IDT			&      & 65.9 & 91.5 \\ 
\textit{Feichtenhofer}~\etal \cite{Feichtenhofer2016} & VGG-16 (RGB+Opt.Flow) + \IDT   &      & 69.2 & 93.5 \\	\hline
\multicolumn{4}{c}{Methods without CNN features} \\  \hline
\textit{Lan~\etal \cite{Lan2015a}} &	\IDT		& 68.0 & 65.4 & 89.1 \\
\textit{Fernando~\etal \cite{Fernando2015}} & \IDT		& \BL{73.7} & 63.7 &   \\
\textit{Hoai \etal \cite{hoi2014}} &	\IDT		& 73.6 & 60.8 &   \\
\textit{Peng \etal \cite{PengECCV2014} } & \IDT       	&    & \BL{66.8} &   \\
\textit{Wu \etal \cite{Wu_2014_CVPR} }   &  \IDT     	&    & 56.4 & 84.2 \\
\textit{Wang \etal \cite{wang2013action} } &   \IDT   	& 64.3 & 57.2 &   \\ \hline

\end{tabular}
\caption{Comparison with the state-of-the-art methods.} 
\label{tab:soa}
\end{table*}
\section{Conclusion}
\label{sec:conclusion}

In this paper we extend the rank pooling method in two ways. First, we introduce an effective, clean, and principled temporal encoding method based on the discriminative rank pooling framework which can be applied over vector sequences or convolutional neural network-based video sequences for action classification tasks. Our temporal pooling layer can sit above any CNN architecture and through a bilevel optimization formulation admits end-to-end learning of all model parameters. We demonstrated that this end-to-end learning significantly improves performance over a traditional rank-pooling approach by 21\% on the UCF-sports dataset and 9.6 mAP on the Hollywood2 dataset.

Secondly, we presented a novel temporal encoding method called \emph{hierarchical rank pooling} which consists of a network of non-linear operations and rank pooling layers. The obtained video representation has high capacity and capability of capturing informative dynamics of rich frame-based feature representations. We also presented a principled way to learn non-linear dynamics using a stack consisting of parametric non-linear activation layers, rank pooling layers, discriminative rank pooling layer and, a soft-max classifier which we coined \emph{discriminative hierarchical rank pooling}. We demonstrated substantial performance improvement over other temporal encoding and pooling methods such as max pooling, rank pooling, temporal pyramids, and LSTMs. Combining our method with features from the literature, we obtained good results on the Hollywood2, HMDB51 and UCF101 datasets.

One of the limitations of our rank pooling-based end-to-end learning is the computational complexity. Especially, the gradient computation of the rank-pooling operator is computationally expensive which limits applicability of end-to-end learning on very large datasets. One solution is to simplify the gradient computation or relax the constraints of the learning objective function as shown in prior work \cite{Bilen2016,bilen2016action}. If one wants to use discriminative rank pooling inside hierarchical rank pooling networks, then perhaps one can find a strategy to reuse the gradient computation of the neighbouring subsequences. These are possible solutions to make the back-propagation faster in our proposed framework.

We believe that the framework proposed in this paper will open the way for embedding other traditional optimization methods as subroutines inside CNN architectures. Our work also suggests a number of interesting future research directions. First, it would be interesting to explore more expressive variants of rank-pooling such as through kernalization. Second, our framework could be adapted to other sequence classification tasks (e.g., speech recognition) and we conjecture that as for video classification there may be accuracy gains for these other tasks too.
\begin{acknowledgements}
This research was supported by the Australian Research Council Centre of Excellence for Robotic Vision (project number CE140100016).
\end{acknowledgements}
%

\end{document}